\documentclass{article}
\usepackage[a4paper,left=3cm,right=3cm,top=3.5cm,bottom=3.cm]{geometry}
\usepackage[title]{appendix}
\usepackage{mathrsfs}  
\usepackage[utf8]{inputenc} % allow utf-8 input
\usepackage[T1]{fontenc}    % use 8-bit T1 fonts

\usepackage{hyperref}       % hyperlinks
\usepackage{url}            % simple URL typesetting
\usepackage{booktabs}       % professional-quality tables
\usepackage{nicefrac}       % compact symbols for 1/2, etc.
\usepackage{microtype}      % microtypography
\usepackage{xcolor}         % colors
\usepackage{ulem}

\usepackage{graphicx}
\usepackage{latexsym}
\usepackage{epsfig}
\usepackage{amsthm}
\usepackage{amssymb}
\usepackage{amstext}
\usepackage{amsgen}
\usepackage{amsxtra}
\usepackage{amsgen}
\usepackage{mathrsfs,mathtools}
\usepackage{bbm}
\usepackage{enumerate,subfigure,color}
\usepackage{tikz}
\usepackage{mathtools}

\newtheorem{thm}{Theorem}[section]
\newtheorem{proposition}[thm]{Proposition}
\newtheorem{lemma}[thm]{Lemma}
\newtheorem{cor}[thm]{Corollary}
\newtheorem{example}[thm]{Example}
\newtheorem{assumption}[thm]{Assumption}

\newcommand{\N}{\mathbb{N}}
\newcommand{\Z}{\mathbb{Z}}
\newcommand{\R}{\mathbb{R}}

\newcommand{\E}{\mathbb{E}}
\newcommand{\Prob}{\mathbb{P}}

\newcommand{\LO}{\mathcal{L}}

\renewcommand{\epsilon}{\varepsilon} % epsilon
\renewcommand{\Im}{\operatorname{Im}}

\newcommand{\sign}{\ensuremath{\mathrm{sign}}} % sign
\DeclareMathOperator*{\argmin}{arg\,min}
\newcommand{\tr}{\operatorname{tr}} % trace
\newcommand{\norm}[1]{\left\|#1 \right\|}

\def\<{\langle}
\def\>{\rangle}
\newcommand{\lY}{\langle}
\newcommand{\lK}{\langle}
\newcommand{\lX}{\langle}
\newcommand{\rY}{\rangle_{Y}}

\newcommand{\rX}{\rangle}
\newcommand{\Se}{\Sigma_{\varepsilon}}

\newcommand{\vz}{z} % an element of Z
\newcommand{\vzv}{{\mathbf{\vz}}} % a vector of elements of Z
\newcommand{\Bs}{B} % frame synthesis operator
\newcommand{\BOp}{B^\star}

\newcommand{\BS}{\widehat{B}}

\newcommand{\wh}[1]{\widehat{#1}}
\newcommand{\wt}[1]{\widetilde{#1}}

\newcommand{\eps}{\varepsilon}

\title{Learning sparsity-promoting regularizers \\ for linear inverse problems}

\author{Giovanni S.~Alberti\footnotemark[1]\thanks{MaLGa Center, University of Genoa, Italy (\{giovanni.alberti, ernesto.devito, matteo.santacesaria\}@unige.it).} \and Ernesto De Vito\footnotemark[1] \and Tapio Helin\thanks{Computational Engineering, Lappeenranta-Lahti University of Technology,
Finland (tapio.helin@lut.fi).} \and Matti Lassas\thanks{Department of Mathematics and Statistics, University of Helsinki, Finland ({matti.lassas@helsinki.fi}).} \and Luca Ratti\thanks{Department of Mathematics, University of Bologna, Italy
({luca.ratti5@unibo.it}).} \and Matteo Santacesaria\footnotemark[1]}
\date{ }

\begin{document}

\maketitle

\begin{abstract}
This paper introduces a novel approach to learning sparsity-promoting regularizers for solving linear inverse problems. We develop a bilevel optimization framework to select an optimal synthesis operator, denoted as \( B \), which regularizes the inverse problem while promoting sparsity in the solution. The method leverages statistical properties of the underlying data and incorporates prior knowledge through the choice of \( B \). We establish the well-posedness of the optimization problem, provide theoretical guarantees for the learning process, and present sample complexity bounds. The approach is demonstrated through theoretical infinite-dimensional examples, including compact perturbations of a known operator and the problem of learning the mother wavelet, and through extensive numerical simulations. This work extends previous efforts in Tikhonov regularization by addressing non-differentiable norms and proposing a data-driven approach for sparse regularization in infinite dimensions.
\end{abstract}

\section{Problem formulation and main contributions}
\label{sec:intro}
Consider a linear inverse problem 
\begin{equation}
y = Ax + \varepsilon,
    \label{eq:invprob}
\end{equation}  
where we assume that $A\colon X \rightarrow Y$ is a linear bounded operator between the real and separable Hilbert spaces $X,Y$, whereas the inverse $A^{-1}$ (if it exists) can be in general an unbounded operator.

We introduce a variational strategy \cite{engl1996} to regularize the inverse problem \eqref{eq:invprob}, and together promote the sparsity of the solution with respect to a suitable basis or frame. To do so, we consider an operator $B \in \mathcal{L}(\ell^2,X)$ (where $\mathcal{L}(\ell^2,X)$ denotes the space of bounded linear operators from the sequence space $\ell^2(\N)$ to the Hilbert space $X$, equipped with the operator norm) and define the following minimization problem:
\begin{equation}
\hat x_B = \Bs \hat u_B, \quad \hat u_B = \argmin_{u \in \ell^2} \Big\{ \frac{1}{2}\| \Se^{-1/2} A\Bs u \|_Y^2 - \lK y, \Se^{-1}A\Bs u \rY + \| u \|_{\ell^1} \Big\},
\label{eq:xhat}
\end{equation}
where $\Se$ is the covariance of the noise $\epsilon$ (see Assumption~\ref{ass:stat} below).
In Section \ref{sec:optimization}, we introduce a set of assumptions that guarantee the well-definedness and well-posedness of problem \eqref{eq:xhat}.

To better interpret the proposed regularization strategy, we remark that in a finite-dimensional setup (let, e.g., $X=\R^n$ and $Y \in \R^{m}$), if the matrix $B \in \R^{n \times n}$ is invertible, the minimization problem \eqref{eq:xhat} admits the following formulation:
\[
\hat x_B = \argmin_{x \in \R^n} \Big\{ \frac{1}{2}\| Ax - y \|^2_{\Se} + \| B^{-1} x \|_{1} \Big\},
\]
where $\| \cdot \|_{\Se} = \| \Se^{-1/2} \cdot \|$ is a norm on $\R^m$ that leverages the knowledge of the noise covariance $\Se$ to whiten the residual error.  It is worth observing that this argument, as well as all the results of this work, can be extended to complex Hilbert spaces $X$ and $Y$, provided that the scalar product in \eqref{eq:xhat} is replaced by its real part; to simplify the exposition we have decided to consider the real case only. Further details about this simplified formulation can be found in Section~\ref{sec:numerics}.
The main focus of this paper is related to the choice of the synthesis operator $B \colon \ell^2 \rightarrow X$ within \eqref{eq:xhat}. In particular, let us introduce the map 
\begin{equation}\label{eq:RB-def}
    R_B\colon Y \rightarrow X, \qquad R_B(y) = \hat{x}_B \quad \text{as in \eqref{eq:xhat}}.
\end{equation}
For a fixed $B$ satisfying the theoretical requirements expressed in Section \ref{sec:optimization}, the map $R_B$ is a \textit{stable} reconstruction operator, i.e., it provides a continuous approximation of the inverse map $A^{-1}$. In particular, $R_B$ promotes the sparsity of the regularized solution in terms of the synthesis operator $B$: namely, $R_B(y) = B \hat{u}_B$ is chosen so that it balances a good data fidelity ($A R_B(y) \approx y$) and only a few components of $\hat{u}_B$ are different from $0$. We observe that the regularization parameter is not explicit in the above formulation, since it is absorbed by $B$.

The choice of the synthesis operator $B$ encodes crucial information regarding the prior distribution of $x$. Consider, for simplicity, a signal processing problem (such as denoising, or deblurring), which, after a discretization of the spaces of signals, can be formulated as a linear inverse problem in $X=Y=\R^n$. Then, promoting the sparsity of the reconstructed signals under the choice $B=\operatorname{Id}$ encodes the prior information that the ground truths are expected to have few, isolated, spikes. Choosing $B$ to be a basis of (discretized) sines and cosines promotes band-limited signals in the Fourier domain. Setting $B$ to a wavelet transform, instead, might promote signals showing few jump discontinuities and smooth anywhere else. In general, the synthesis operator $B$ should be carefully chosen to incorporate, in the reconstruction operator $R_B$, any significant prior knowledge, or prior belief, on the ground truths.

In this paper, we follow a statistical learning approach for the selection of $B$. In particular, we assume that $x$ and $y$ are random objects with a joint probability distribution $\rho$. We then leverage the (partial) knowledge of such a probability distribution to define a data-driven rule for selecting $B$. To do so, we must introduce some assumptions on the probability distribution $\rho$, which are carefully detailed in Section \ref{sec:stat}. Before delving into the details, though, let us briefly sketch the overall idea of the learning-based choice of $B$, introducing some minimal requirements on the random objects $x,y$.
\begin{assumption} \label{ass:stat}
Let $(x,y) \sim \rho$, being $y = Ax+\varepsilon$, and where:
\begin{itemize}
    \item $x$ is a square-integrable random variable on $X$; 
    \item $\varepsilon$ is a zero-mean square-integrable random variable on $Y$, independent of $x$, with known and injective covariance $\Se\colon Y \rightarrow Y$.
\end{itemize}
\end{assumption}

We recall that assuming $\eps$ to be a zero-mean square-integrable random variable (i.e., $\E{[\norm{\eps}^2_Y]}<+\infty$) implies that its covariance operator $\Se$ is positive and trace-class (or nuclear, see \cite[Theorem 1.7]{bosq2000linear}), namely, that the sequence of  eigenvalues is summable. For this reason,
such an assumption is in principle quite restrictive when $Y$ is infinite-dimensional and, for example, the case of white noise requires careful treatment. In its most common formulation, white noise is modeled as a random process on a Hilbert space (for example, $Y=L^2(\Omega)$) with zero mean and covariance equal to the identity. Such a noise model clearly fails in satisfying the hypotheses whenever $Y$ is infinite dimensional, because the covariance $\Se = \operatorname{Id}$ is not trace-class, but only bounded. However, it is possible to represent such a process as a square-integrable random variable by carefully selecting the space $Y$, as described in Appendix~\ref{app:white_noise}.

For a specific choice of synthesis operator $B$, the quality of the reconstruction operator $R_B$ can be evaluated through the \textit{expected loss} $L\colon \mathcal{L}(\ell^2,X) \rightarrow \R$:
\begin{equation}
L(B) = \mathbb{E}_{(x,y)\sim \rho}[\|R_B(y)-x\|_X^2].
    \label{eq:loss}
\end{equation}
Let us now consider a suitable class $\mathcal{B} \subset \mathcal{L}(\ell^2,X)$ of \textit{admissible} operators, namely, that satisfy the assumptions reported in Section \ref{sec:optimization} guaranteeing the existence and uniqueness of the solution of \eqref{eq:xhat}. We define the optimal regularizer $R^\star = R_{B^\star}$, where $B^\star$ is a minimizer of the expected loss over the set of admissible operators $\mathcal B$, whose existence follows under suitable conditions on $\mathcal B$. In particular, putting together the expressions of \eqref{eq:loss} and \eqref{eq:xhat}, we obtain: 
\begin{equation} \label{eq:bilevel}
    \begin{gathered}
        B^\star \in \argmin_{B\in\mathcal B} 
        L(B),
        \\
        R_B(y) = B \argmin_{u \in \ell^2} \Big\{ \frac{1}{2}\| \Se^{-1/2} A\Bs u \|_Y^2 - \lK y, \Se^{-1}A\Bs u \rY + \| u \|_{\ell^1} \Big\},
    \end{gathered}    
\end{equation}
which is usually referred to as a \textit{bilevel} optimization problem. Note that the optimality of $B^\star$ strongly depends on the choice of the class $\mathcal B$. We moreover stress that the optimal target $\BOp$ can only be computed if the joint probability distribution $\rho$ of $x$ and $y$ is known. This is not verified in practical applications, but in many contexts, we may suppose to have access to a sample of $m$ pairs $\vzv= \{(x_j,y_j)\}_{j=1}^m$ such that the family is independent and identically distributed as $(x,y)$. 
In this case, following the paradigm of \textit{supervised learning}, we can approximate the expected loss by an empirical average, also known as \textit{empirical risk}, namely
\begin{equation}
    \label{eq:risk}
    \widehat{L}(B) = \frac{1}{m} \sum_{j=1}^m \| R_{B} (y_j) - x_j \|_{X}^2.
\end{equation}
A natural estimator of $B^\star$ is then given by any minimizer $\BS$ of the empirical loss:
\begin{equation}
  \BS \in \argmin_{B\in\mathcal B} \wh{L}(B).
    \label{eq:empiricaltarget}
\end{equation}
We stress that both $\wh{L}(B)$ and $\BS$ are random variables depending on the sample $\vzv$. We simply denote this dependence by $\wh{\cdot}$. The task of \textit{sample error estimates} is to quantify the dependence of the empirical target $\BS$ on the sample $\vzv$, in particular by bounding (either in probability or in expectation) the \textit{excess risk} $L(\BS) - L(B^\star)$ in terms of the sample size $m$.

\subsection*{Motivation}  

The theoretical framework developed in this paper allows significant flexibility in choosing the class $\mathcal{B}$ of admissible operators, enabling various practical applications. 
In this paper we provide two examples, detailed in Section \ref{sec:examples}, of our theory in an infinite-dimensional setting. In particular, we establish concrete guarantees on the sample complexity of the learning problem by deriving explicit bounds on the covering numbers of the respective operator classes.
In the first example, we assume that some prior knowledge about an effective synthesis operator $B_0$ is available (e.g., from physical understanding of the problem). In this setting, we construct a class $\mathcal{B}$ of compact perturbations around this reference operator. This approach maintains the desirable properties of $B_0$ while allowing for data-driven refinements. The second example is in the framework of signal processing applications where wavelets provide powerful sparse representations for many natural signals exhibiting multi-scale features. Our framework enables learning an optimal mother wavelet from data, rather than selecting from predefined wavelet families.

 The examples are complemented by the numerical studies of Section \ref{sec:numerics}, demonstrating how the abstract mathematical formulation directly connects to practical applications in finite-dimensional spaces.

\subsection*{Main contribution of this paper}  

In our analysis, we pursue the following main goals:
\begin{enumerate}
    \item studying the theoretical properties of the inner minimization problem \eqref{eq:xhat}, and in particular its well-posedness for fixed $B$ and the continuous dependence of the minimizer $\hat{x}_B$ in terms of $B \in \mathcal{B}$ (see Theorems \ref{thm:xhat_is_unique} and \ref{thm:stab} in Section \ref{sec:optimization});
    \item 
    studying the approximation properties of the empirical target $\BS$ (constructed leveraging the knowledge of $A$, $\Se$, and of the sample $\vzv$) with respect to the optimal target $B^\star$, deriving sample error estimates for the excess risk (see Theorem \ref{thm:CuckSmale} in Section \ref{sec:stat});
    \item formulate some relevant applications which satisfy the assumptions on $x$, $y$, $A$, and $\mathcal{B}$ introduced in Sections \ref{sec:optimization} and \ref{sec:stat} (see Section \ref{sec:examples});
    \item discuss numerical strategies to provide an approximation of the empirical target $\widehat{B}$ and showcase their effectiveness on some benchmark applications (see Section \ref{sec:numerics}).
\end{enumerate}

\subsection*{Comparison with existing literature}  

This work originated from the analysis carried out by a subset of the authors in \cite{alberti2021learning}, where we considered the problem of learning (also) the optimal operator $B$ in generalized Tikhonov regularization, i.e., when a quadratic penalty term is considered within the inner problem \eqref{eq:xhat} instead of the $\ell^1$ norm. The main difficulty associated with the proposed extensions resides in the lack of strong convexity and differentiability of the $\ell^1$ norm. Moreover, unlike in the Tikhonov case, the inner minimization problem \eqref{eq:xhat} does not possess a closed-form solution: unfortunately, this also prevents the explicit computation of the optimal regularizer $B^\star$, which was one of the main results in \cite{alberti2021learning}. As a final consequence, it is not possible to formulate here an \textit{unsupervised} strategy to learn $B^\star$, i.e., based only on a training set of ground truths $\{x_j\}_{j=1}^m$: indeed, the unsupervised strategy proposed in \cite{alberti2021learning} extensively leveraged the explicit expression of the optimal operator $B^\star$.

Other extensions of the work \cite{alberti2021learning} have been carried out in \cite{ratti2023learned,burger2023learned,chirinos2024learning,brauer2024learning,alberti2024learning} even though a statistical learning approach for sparse optimization in infinite dimension was not considered yet. See also \cite{hauptmann2024convergent} for a connection between generalized Tikhonov and linear plug-and-play denoisers. Let us just mention that bilevel approaches for inverse problems in imaging have been studied, from numerical and optimization points of view since many years \cite{calatroni2017bilevel,de2017bilevel}. In particular, the works \cite{horesh-haber-2009,huang-haber-horesh-2012,2022-ghosh-etal,ghosh-etal-2024} study the bilevel learning of $\ell^1$ regularizers, focusing on the finite-dimensional case and mostly on the algorithmic aspects, without considering the generalization issue from the theoretical point of view.  Moreover, our work is also related to statistical inverse learning for inverse problems \cite{blanchard2018optimal,abhishake2024statistical, bubba2023convex} and, also, more generally to the growing field of operator learning \cite{nelsen2021random, lanthaler2022error,kovachki2023neural,de2023convergence,boulle2023mathematical}. Let us also mention some of the main works on using machine learning techniques for solving inverse problems that had motivated this work \cite{adler2017solving,2017-kobler-etal,adler2018learned,lunz2018adversarial,arridge2019,li2020nett,mukherjee2020learned,mukherjee2021adversarially}. 

It is worth observing that the problem of learning an operator $B$ yielding a sparse representation of a dataset $\{x_j\}_{j=1}^m$ is deeply connected with the well-known task of Dictionary Learning. Although it is possible to provide a formulation of a dictionary learning problem very close to the bilevel problem \eqref{eq:bilevel}, the two problems pursue two distinct aims and may lead to different results, as shown in Appendix \ref{app:DL}. Despite this, the sample complexity bounds we obtain in this work are comparable to those derived for dictionary learning \cite{2015-gribonval-etal,2022-sulam-etal}.

\section{Theoretical results on the deterministic optimization problem \texorpdfstring{\eqref{eq:xhat}}{(1.2)}}
\label{sec:optimization}

The goal of this section is to study the well-posedness of the minimization problem formulated in \eqref{eq:xhat} for a fixed $B \in \mathcal{L}(\ell^2,X)$. Moreover, we study the stability of the minimization with respect to perturbations of $B$. To do so, we need to introduce a suitable set of hypotheses.

For the rest of the paper, we use the convention $J_B(u) = F_B(u) + \Phi(u)$, where
\begin{equation} \label{eq:FB}
    F_B(u) = \frac{1}{2}\| \Se^{-1/2} A\Bs u \|_Y^2 - \lK y, \Se^{-1} A\Bs u \rY
    \quad \text{and} \quad \Phi(u) = \| u \|_{\ell^1}.
\end{equation}

\begin{assumption}[Compatibility assumption between $\Se$ and $A$]\label{ass:compatibility}
Let the covariance matrix $\Se$ satisfy Assumption \ref{ass:stat}, and assume $\operatorname{Im}(A) \subset\operatorname{Im}(\Se)$. Moreover, we assume that
    \begin{equation}
    \Se^{-1} A \colon X \rightarrow Y \text{ is a compact operator}.
    \label{eq:compactibility}
    \end{equation}
\end{assumption}

Assumption \ref{ass:compatibility} is a technical requirement that is needed to guarantee that the functional $J_B$ is well-defined for a fixed $B \in \mathcal{L}(\ell^2,X)$.
Indeed, the first term in \eqref{eq:FB} can be rewritten as $\displaystyle \frac{1}{2} \lK \Se^{-1}A\Bs u, A\Bs u \rY$, and $\Se^{-1}A\Bs u$ exists and is unique since $\operatorname{Im}(A\Bs) \subset \operatorname{Im}\Se = \operatorname{dom}(\Se^{-1})$, and analogously for the second term. Finally, since the minimization problem \eqref{eq:xhat} is set in the Hilbert space $\ell^2 \supset \ell^1$, the third term in \eqref{eq:xhat} should be interpreted as $\| u\|_{\ell^1}$ if $u \in \ell^1$ and $+\infty$ if $u \in \ell^2 \setminus \ell^1$. We moreover remark that Assumption~\ref{ass:compatibility} is always satisfied in finite dimension when the covariance $\Se$ is invertible.

Since $J_B$ is non-differentiable and not strictly convex, further assumptions on the interplay between $A$ and $B$ are needed to guarantee the well-posedness of the minimization task. Our analysis is confined to the case where
$AB$ is assumed to be finite basis injective.

\begin{assumption}[Finite Basis Injectivity (FBI) of $AB$]\label{ass:FBI} 
For all $I\subset \N$ with $\operatorname{card}(I) < \infty$, the operator 
$(AB)|_{\ell^2_I }$ is injective, where we denoted $\N = \{1,2,\dots\}$ and $\ell^2_I := \operatorname{span}\{ e_i: i \in I\}$. 
\end{assumption}
Note that while this assumption is satisfied when $B$ is injective (provided $A$ is injective), it also holds for more general structures (e.g.\ FBI frames, see \cite{bredies2008linear}). We remark that, in order to guarantee the well-posedness of \eqref{eq:xhat}, there exist some alternatives to Assumption \ref{ass:FBI}, such as the Restricted Isometry Property (see \cite{candes2005decoding}), the source conditions (see \cite{grasmair2011necessary}), or the strict sparsity pattern (see \cite{bredies2008linear}). Compared to these alternatives, Assumption \ref{ass:FBI} allows us to consider more general operators $A$, to provide a clearer interpretation, and to avoid introducing requirements on the unknown ground truth, respectively.

The well-posedness of the minimization problem \eqref{eq:xhat} is now characterized by the following theorem.

\begin{thm}
\label{thm:xhat_is_unique}
Let $A \in \mathcal{L}(X,Y), \Se \in \mathcal{L}(Y,Y)$, and $B \in \mathcal{L}(\ell^2,X)$ satisfy Assumptions \ref{ass:compatibility} and \ref{ass:FBI}, and let $y \in Y$.
There exists a unique minimizer $\hat u_B = \hat u_B(y)$ of $J_B$, and, consequently, a unique  $\hat x_B = B \hat u_B$ in \eqref{eq:xhat}.
\end{thm}

The proof is presented in two parts: the existence is proved at the end of Section \ref{subsec:existence} while the uniqueness follows as a consequence of Theorem \ref{thm:stab} at the end of Section \ref{subsec:stability}.

Let us next consider the stability of the minimizer with respect to perturbing $B$ in a suitable class of operators $\mathcal{B}$. Towards that end, let us introduce the following assumption, which is also strongly motivated by the statistical learning results presented in Section \ref{sec:stat}, and has also been considered in \cite{alberti2021learning}.
\begin{assumption}[Requirements on $\mathcal{B}$] \label{ass:mathcalB}
Let $\mathcal{B} \subset \mathcal{L}(\ell^2,X)$ be a compact set of operators such that every $B \in \mathcal{B}$ satisfies Assumption \ref{ass:FBI}. 
\end{assumption}

A first consequence of the compactness of $\mathcal{B}$ is that the norms $\| B \|$, $B\in\mathcal{B}$, are uniformly bounded by $|{\mathcal B}| := \sup_{B\in{\mathcal B}} \norm{B} < \infty$.
Moreover, as we show in the following subsections, Assumption \ref{ass:mathcalB} allows us to provide a uniform expression of several properties of the solutions of problem \eqref{eq:xhat}, namely, independently of the choice of $B$. The most relevant consequence of this discussion is the following result, providing a global stability estimate of the minimizers $\hat{x}_B$ with respect to perturbations of $B$ within $\mathcal{B}$.

\begin{thm} \label{thm:stab}
Let $A,\Se$ and $\mathcal{B}$ satisfy Assumptions \ref{ass:stat}, \ref{ass:compatibility} and \ref{ass:mathcalB}, and let $\| y\|_Y \leq Q$. Then, there exists a constant $c_{\rm ST} = c_{\rm ST}(\mathcal{B},A,\Se,Q)$ such that for every $B_1,B_2 \in \mathcal{B}$, 
\begin{equation}
\begin{gathered}
\|R_{B_1}(y)-R_{B_2}(y) \|_{X} \leq c_{\rm ST} \| B_1 - B_2 \|,
\end{gathered}
    \label{eq:stab}
\end{equation}
\end{thm}

\subsection{Existence of the minimizer with fixed \texorpdfstring{$B \in {\mathcal B}$}{B in B}}
\label{subsec:existence}

To establish the existence of a minimizer $\hat{u}_B$ in \eqref{eq:xhat}, we first show that the sublevel sets of $J_B$ are bounded, which ensures the weak convergence of any minimizing sequence.

\begin{lemma}
\label{lem:aux_lb}
Let $A,\Se$ and $\mathcal{B}$ satisfy Assumptions \ref{ass:compatibility} and \ref{ass:mathcalB}. There exists a monotonically increasing function $h \colon \R_+ \to \R_+$ depending on $A$, $\Se$, and $|\mathcal{B}|$ such that the following holds: if  $\norm{w}_X \leq |{\mathcal B}|$ and $\norm{\Se^{-1}Aw}_Y \geq \gamma>0$, then $\norm{\Se^{-1/2}Aw}_Y \geq h(\gamma)>0$.
\end{lemma}

\begin{proof}
Let us prove existence of a general $h$ by contradiction: assume that for all $n$ there exists $w_n \in X$ satisfying $\norm{w_n}_X \leq |{\mathcal B}|$, $\norm{\Se^{-1}Aw_n}_Y \geq \gamma$ and $\norm{\Se^{-1/2}Aw_n}_Y < \frac 1n$. Since $\{w_n\}_{n=1}^\infty \subset X$ is bounded, we can find a subsequence $\{w_{n_k}\}_{k=1}^\infty$ that weakly converges to some $w\in X$. 
Due to the weak lower semicontinuity of the norm (see, e.g., \cite[Section 1.4 and Corollary 3.9]{brezis2011functional}), we have that $\| \Se^{-1/2}A w\|_Y \leq \liminf_{k} \| \Se^{-1/2}Aw_{n_k}\|_Y = 0$, and also, by the injectivity of $\Se^{-1/2}$, that $w\in \operatorname{ker}(A)$. Therefore, by the compactness of $\Se^{-1}A$ we have that $\| \Se^{-1}A w_{n_k}\|_{Y} \rightarrow 0$.
This contradicts our assumption and, therefore, a lower bound $h(\gamma) >0$ must exist.

The existence of a monotonically increasing $h$ can also be demonstrated by contradiction: suppose that there exist $\gamma_1<\gamma_2$ such that for any $h$ satisfying the statement, we have $h(\gamma_1)>h(\gamma_2)$. Then we immediately see that $\tilde h$ defined by
\begin{equation*}
    \tilde h(\gamma) = 
    \begin{cases}
        h(\gamma)& \gamma\neq \gamma_1 \\
        h(\gamma_2)& \gamma = \gamma_1,
    \end{cases}
\end{equation*}
also satisfies the claim yielding the contradiction.
\end{proof}

\begin{proposition} \label{prop:sublevels}
Let $M \in \R$, and suppose Assumptions \ref{ass:stat}, \ref{ass:compatibility}, \ref{ass:FBI} and \ref{ass:mathcalB} hold. Let $\| y \|_Y \leq Q$. Then, there exists a constant $c_{\rm UB}(M) = c_{\rm UB}(M; A,\Se, Q,|{\mathcal B}|)$ such that, for every $B \in \mathcal{B}$, we have
\begin{equation} \label{eq:sublevel}
     \{u \in \ell^2 \; | \; J_B(u) \leq M\} \subset  \{u \in \ell^2 \; | \; \| u \|_{\ell^1} \leq c_{\rm UB}(M)\}.
\end{equation}
\end{proposition}
\begin{proof}

The condition $J_B(u)\leq M$ reads as
\begin{equation*}
\frac{1}{2} \| \Se^{-1/2} A B u \|_Y^2 + \norm{u}_{\ell^1} \leq M + \langle y,\Se^{-1}A Bu \rY.
\end{equation*}
Let us next make a change of variables with $w = Bu/\norm{u}_{\ell^1}$ and $\tau = \norm{u}_{\ell^1}$, which leads to
\begin{equation}
\label{eq:aux_prop_bound2}
\frac{1}{2} \| \Se^{-1/2} A w \|_Y^2\tau ^2 + \tau \leq M + \tau \norm{y}_Y \norm{\Se^{-1}A w}_Y.
\end{equation}
In particular, we have
\begin{equation*}
1  \leq \frac M \tau + \norm{y}_Y \norm{\Se^{-1}A w}_Y.
\end{equation*}
Hence, either $\norm{u}_{\ell^1} = \tau \leq \max\{2M,0\}$ or for any $\tau \geq \max\{2M,0\}$ it holds that
\begin{equation*}
    \norm{\Se^{-1}A w}_Y \geq \frac{1}{2 \norm{y}_Y}.
\end{equation*}
Since $\norm{Bu}_X \leq |\mathcal{B}|\norm{u}_{\ell^2}$ for all $u \in \ell^2$ and $B \in \mathcal{B}$, we have
\begin{equation*}
\| u \|_{\ell^1} \geq \| u \|_{\ell^2} \geq \frac{1}{|{\mathcal B|}} \| Bu \|_X \qquad \forall u \in \ell^2, \forall B \in \mathcal{B},
\end{equation*}
which also implies that $\norm{w}_X \leq |{\mathcal B}|$.
Next, by Lemma \ref{lem:aux_lb} it follows that
\begin{equation*}
    \norm{\Se^{-1/2} A w}_Y \geq h,
\end{equation*}
where we abbreviate $h= h\left(\frac{1}{2 \norm{y}_Y}\right)$ for convenience.
Modifying \eqref{eq:aux_prop_bound2} we obtain
\begin{equation*}
    \frac{1}{2} h^2\tau^2  \leq M + \tau \norm{y}_Y \norm{\Se^{-1}A} |\mathcal{B}| \leq M + \frac 14 h^2\tau^2 + \frac{\norm{y}_Y^2 \norm{\Se^{-1}A}^2 |\mathcal{B}|^2}{h^2}
\end{equation*}
and solving for $\tau$ yields
\begin{equation}
    \label{eq:sublevel_tau_bound}
    \tau^2 \leq \frac{4M}{h^2} + \frac{4\norm{y}_Y^2 \norm{\Se^{-1}A}^2 |\mathcal{B}|^2}{h^4}.
\end{equation}
Note that, by the monotonicity of $h$, from $\| y\|_Y \leq Q$ it follows that  $h\geq h\left(\frac{1}{2Q}\right)$, hence 
\begin{equation}
    \label{eq:sublevel_tau_bound_Q}
    \tau^2 \leq \frac{4M}{h\left(\frac{1}{2Q}\right)^2} + \frac{4Q^2 \norm{\Se^{-1}A}^2 |\mathcal{B}|^2}{h\left(\frac{1}{2Q}\right)^4}.
\end{equation}
Now the desired bound for $\tau = \norm{u}_{\ell^1}$ follows for any $u \in \ell^2$, $B \in \mathcal{B}$.
\end{proof}

\begin{proof}[Proof of Theorem~\ref{thm:xhat_is_unique} (existence)]
The existence of a minimizer $\hat{u}_B$ in \eqref{eq:xhat} is now guaranteed by standard arguments: any minimizing sequence $\{u_j\}_{j=1}^\infty \subset \ell^2$ belongs to some sublevel set and, therefore, to some $\ell^1$-ball according to Proposition \ref{prop:sublevels}. By the Banach--Alaoglu theorem, the sequence has a weak-* converging subsequence in $\ell^1$, namely, there exists $\{u_{j_l}\}_{l=1}^\infty \subset \ell^1$ such that $u_{j_l}\xrightharpoonup{\ast} \bar{u}$. Since $F_B$ is continuous and the $\ell^1$-norm is lower semicontinuous with respect to the weak-* topology (i.e., $u_{j_l}\xrightharpoonup{\ast} \bar{u}$ implies $ \| \bar{u} \|_{\ell^1} \leq\liminf_{l \rightarrow \infty} \| u_{j_l}\|_{\ell^1}$), one can show that the limit is a minimizer of $J_B$.
\end{proof}

Let us note that since $J_B(\hat{u}_B) \leq J_B(0) = 0$, denoting by $c_{\rm UB}$ the constant $c_{\rm UB}(0)$ in \eqref{eq:sublevel} we have the following bound for such minimizers, uniformly in $B \in \mathcal{B}$.
\begin{equation}
\|\hat{u}_B\|_{\ell^2} \leq \|\hat{u}_B\|_{\ell^1} \leq c_{\rm UB}.
\label{eq:boundmin}    
\end{equation}
The uniqueness of such a minimizer is a consequence of Proposition \ref{prop:condition}, as we will show later.

\subsection{Stability in \texorpdfstring{${\mathcal B}$}{B} and uniqueness of the minimizer}
\label{subsec:stability}

The key result of this section is Proposition \ref{prop:condition}, as it directly enables the uniqueness of the minimizer in Theorem \ref{thm:xhat_is_unique}, and its stability with respect to $B$. Such a result is in the spirit of \cite[Theorem 2]{bredies2008linear}, and its proof is partially based on the one of \cite[Lemma 3]{bredies2008linear}. The main contribution of Proposition \ref{prop:condition} is that the constant appearing in \eqref{eq:condition} does not depend on the choice of $B$, provided that we consider an operator class ${\mathcal B}$ satisfying Assumption \ref{ass:mathcalB}.
The proof of Proposition \ref{prop:condition} requires some preliminary lemmas.

Let us first note that $F_B \colon \ell^2 \rightarrow \R$ is convex and differentiable with gradient
\[
F_B'(u) = B^*\left(\Se^{-1}A\right)^* (AB u - y),
\]
where we have used that $A^* (\Se^{-1} A) = \left(\Se^{-1}A\right)^* A$, since 
\(\langle \Se^{-1} Av, Aw \rY = \langle  Av, \Se^{-1}Aw \rY = \langle  (\Se^{-1}A)^*Av, w \rX,\) 
for all $v,w \in X$ as $\Se^{-1}$ is self-adjoint 
and $\Im(A) \subset \operatorname{dom}({\Se^{-1}})$.
Notice that $F_B'$ is affine in $u$ and Lipschitz continuous, and the Lipschitz constant is uniformly bounded by $L_{F'} = |{\mathcal B}|^2 \| \Se^{-1/2} A \|^2$ for $B \in {\mathcal B}$. 
Since both $F_B$ and $\Phi$ are convex, $\hat{u}_B$ being a minimizer of $J_B$ is equivalent to $0$ belonging to the subgradient set of $J_B$ at $\hat{u}_B$, i.e. $0\in \partial J_B(\hat{u}_B)$.
Now, due to the differentiability of $F_B$, an equivalent optimality condition is given by $\hat{w}_B \in \partial \Phi(\hat{u}_B)$, where
\begin{equation}
\hat{w}_B = -F'_B(\hat{u}_B) = B^*\left(\Se^{-1}A\right)^* (y- AB \hat{u}_B).
    \label{eq:OC}
\end{equation}
Moreover, thanks to the explicit expression of the subdifferential of the $\ell^1$-norm, we can specify that $\hat{w}_B \in \partial \Phi(\hat{u}_B)$ is equivalent to
\begin{equation}
 \hat{w}_{B,k} = \left\{ \begin{aligned}
\sign(\hat{u}_{B,k}) \quad &\text{if } \hat{u}_{B,k} \neq 0 \\
\xi \in [-1,1]  \quad &\text{if } \hat{u}_{B,k} = 0. \\
\end{aligned}\right.
\label{eq:ell1subdiff}    
\end{equation}

The following lemmas explore properties that  are valid uniformly in the compact operator class $\mathcal{B}$. For $N > 0$, we introduce the orthogonal projection operator onto $\ell^2$ by setting
\begin{equation*}
    P_N\colon \ell^2 \rightarrow \ell^2: \ P_N u = (u_1, \ldots, u_N, 0, 0, \ldots), \qquad P_N^\perp = \operatorname{Id} - P_N.
\end{equation*}

\begin{lemma}\label{lem:PN1}
Suppose that Assumptions \ref{ass:compatibility}, \ref{ass:FBI}, and \ref{ass:mathcalB} are satisfied. Let $\mathbb{B}_{Y}(Q) = \{y \in Y: \| y\|_Y \leq Q\}$. Then the elements $\hat{w}_B \in \ell^2$ corresponding to the minimizers $\hat{u}_B \in \ell^2$ of $J_B$ via \eqref{eq:OC} satisfy
\begin{equation*}
\lim_{N \rightarrow \infty} \sup\big\{ 
\| P_N^\perp \hat{w}_B \|_{\ell^2}: \; B \in \mathcal{B},\ y \in \mathbb{B}_{Y}(Q) \big\} = 0.
\end{equation*}
\end{lemma}
\begin{proof}
By contradiction, suppose there exist $\eta > 0$ and a diverging sequence $N_M$, $M=1,2,...$, such that 
\[
 \sup\bigl\{ 
\big\| P_{N_M}^\perp B^* \big( \Se^{-1}A \big)^*\big(  y - A B \hat{u}_B \big) \big\|_{\ell^2}:  B \in \mathcal{B},\ y \in \mathbb{B}_{Y}(Q) \bigr\} \geq 2\eta \quad \forall M.
\]
Consider in particular a sequence $\{B_M\}_{M=1}^\infty \subset \mathcal{B}$ and a sequence $\{y_M\}_{M=1}^\infty \subset \mathbb{B}_Y(Q)$ satisfying
\[
\big \| P_{N_M}^\perp B_M^* \big( \Se^{-1}A \big)^*\big( y_M - A B_M \hat{u}_{B_M} \big) \big\|_{\ell^2}  \geq \eta \quad \forall M,
\]
being $\hat{u}_{B_M}$ the solution of \eqref{eq:xhat} with $B=B_M$ and $y=y_M$.
By the compactness of $\mathcal{B}$,  there exists $B \in \mathcal{B}$ such that $ B_M \rightarrow B$ in the strong operator topology, up to a subsequence. Moreover, consider the sequence $h_M = A B_M \hat{u}_{B_M}$, $M=1,2,...$. Since $\norm{\hat{u}_{B_M}}$ is bounded uniformly with respect to $B\in \mathcal{B}$ and $y \in \mathbb{B}_Y(Q)$ by \eqref{eq:boundmin}, $\norm{B_M} \leq |{\mathcal B}|$ independently of $M$, and $A$ is bounded, the sequence $\{h_M\}_{M=1}^\infty\subset Y$ is bounded, thus weakly convergent (up to a subsequence) to an element $h \in Y$. As a consequence of the compactness of $\big(\Se^{-1} A\big)^*$ by the Schauder theorem, we have that $(\Se^{-1} A)^*h_M \rightarrow (\Se^{-1} A)^* h$ in $X$. 
Similarly, the bounded sequence $\{y_M\}_{M=1}^\infty$ admits a weak limit $y \in Y$, and $(\Se^{-1} A)^*y_M \rightarrow (\Se^{-1} A)^* y$.

In conclusion, we have that
\begin{eqnarray*}
\eta &\leq & \| P_{N_M}^\perp B_M^*(\Se^{-1} A)^*(y_M - h_M)\|_{\ell^2}  \\
& \leq & \| P_{N_M}^\perp B_M^*(\Se^{-1} A)^*(y_M - y)\|_{\ell^2} + \| P_{N_M}^\perp B_M^*(\Se^{-1} A)^*(h_M - h)\|_{\ell^2}\\
& &\quad + \| P_{N_M}^\perp B_M^*(\Se^{-1} A)^*(y - h)\|_{\ell^2}
\\
& \leq & |\mathcal{B}| \| (\Se^{-1} A)^*(y_M - y)\|_{X} + |\mathcal{B}| \| (\Se^{-1} A)^*(h_M - h)\|_{X} \\
& &\quad+ \| P_{N_M}^\perp (B_M-B)^*(\Se^{-1} A)^*(y - h)\|_{\ell^2} 
+ \| P_{N_M}^\perp B^*(\Se^{-1} A)^*(y - h)\|_{\ell^2}  \\
& \leq & |\mathcal{B}| \| (\Se^{-1} A)^*(y_M - y)\|_{X} + |\mathcal{B}| \| (\Se^{-1} A)^*(h_M - h)\|_{X}  \\
 & & \quad+ \| B_M- B \|_{\ell^2 \rightarrow  X} \|(\Se^{-1} A)^*(y - h)\|_{X} +  \| P_{N_M}^\perp B^*(\Se^{-1} A)^*(y - h)\|_{\ell^2}.
\end{eqnarray*}
As $M \rightarrow \infty$, all the terms on the right-hand side converge to $0$, which entails a contradiction.
\end{proof}

\begin{lemma}
Suppose that Assumptions \ref{ass:compatibility}, \ref{ass:FBI}, and \ref{ass:mathcalB} are satisfied. Take $N\in\N$. Then, there exists a constant $c_{\rm LB} = c_{\rm LB}(A,\mathcal{B},N)>0$ such that
\begin{equation*}
\| \Se^{-1/2} ABP_N u \|_{Y}^2 \geq c_{\rm LB} \| P_N u \|_{\ell^2}^2
\end{equation*}
for all $B \in {\mathcal B}$ and $u\in\ell^2$.
\label{lem:PN2}
\end{lemma}
\begin{proof}
We prove this claim by contradiction. Assume that there exist two sequences $\{u_M\}_{M=1}^\infty \subset \ell^2$ and $\{B_M\}_{M=1}^\infty \subset \mathcal{B}$ such that
\begin{equation}
\label{eq:lower_bound_aux1}
\frac{\| \Se^{-1/2} A B_M P_N u_M\|_{Y}^2}{\|P_N u_M \|_{\ell^2}^2} < \frac{1}{M} \quad \forall M.
\end{equation}
Let us write $p_M = \frac{P_N u_M}{\|P_N u_M\|_{\ell^2}}$ to obtain that a sequence $\{p_M\}_{M=1}^\infty \subset H_N$, where $H_N = \{ u \in \ell^2: u_k = 0 \ \forall k >N, \|u \|_{\ell^2}=1\}$. Notice that $H_N\subset X$ is a finite-dimensional, bounded and closed subset and hence compact. Rephrasing \eqref{eq:lower_bound_aux1} we have 
\[
\| \Se^{-1/2} A B_M p_M\|_{Y}^2 < \frac{1}{M} \quad \forall M.
\]
Since both $\{p_M\}_{M=1}^\infty\subset H_N$ and $\{B_M\}_{M=1}^\infty \subset {\mathcal B}$ belong to compact sets, up to a subsequence, the sequences converge to limit points $p \in H$ and $B \in \mathcal{B}$, respectively. By the continuity of $\Se^{-1/2}A$ and by the equi-continuity of the operators in $\mathcal{B}$, we deduce that $\| \Se^{-1/2} A B p \|_Y = 0$, which by the injectivity of $\Se^{-1/2}$ and by Assumption \ref{ass:FBI} implies that $p = 0$. This yields a contradiction with the assumption that $\|p\|=1$ and completes the proof.
\end{proof}

We can finally state and prove the main result of this subsection:
\begin{proposition} \label{prop:condition}
Let $A,\Se$ and $\mathcal{B}$ satisfy Assumptions \ref{ass:stat}, \ref{ass:compatibility} and \ref{ass:mathcalB}, and let $\| y\|_Y \leq Q$.
Let $\hat{u}_B \in \ell^2$ be a minimizer of $J_B$ being $B\in \mathcal{B}$ and $y \in \mathbb{B}_Y(Q)$. For every $M$, there exists a constant $\tilde c_{\rm ST} = \tilde c_{\rm ST}(M;A,\Se,\mathcal{B},Q)$ such that for any $B\in{\mathcal B}$ we have
\begin{equation} \label{eq:condition}
    \{v \in \ell^2 \; | \; J_B(v) \leq M\} \subset 
    \{v \in \ell^2 \; | \; \| v - \hat{u}_B \|_{\ell^2}^2 \leq \tilde c_{\rm ST} (J_B(v) - J_B(\hat{u}_B))\}.
\end{equation}
\end{proposition}

\begin{proof}
For clarity, the proof is divided into several steps.

\textit{Step $1$.}
Let us show that there exists $N_0 = N_0(A,\Se,\mathcal{B},Q)$ such that $P_{N_0}^\perp \hat{u}_B = 0$.
Rephrasing Lemma \ref{lem:PN1}, we have that there exists a sequence $\delta_N$, $N=1,2,...,$ converging to zero such that 
\[
\| P_N^\perp\hat{w}_B \|_{\ell^2} \leq \delta_N \quad \forall B \in \mathcal{B}, \ y \in \mathbb{B}_Y(Q).
\]
Therefore, it is possible to pick $N_0$ such that
\[
\| P_{N_0}^{\perp} \hat{w}_B \|_{\ell^2} \leq \frac{1}{2} 
\]
for all $B\in {\mathcal B}$ and $y \in \mathbb{B}_Y(Q)$. Notice that such $N_0$ is independent of the specific choice of $B$, and depends on $y$ only through $Q$: thus, we denote it as $N_0(A,\Se,\mathcal{B},Q)$.
Since $\| P_{N_0}^{\perp} \hat{w}_B \|_{\ell^\infty} \leq \| P_{N_0}^{\perp} \hat{w}_B \|_{\ell^2}$, we also deduce that
\begin{equation}
|\hat{w}_{B,k} | := \left|[\hat{w}_B]_k\right| \leq \frac{1}{2} \quad \text{for $k > {N_0}$,}
\label{eq:wkPN}    
\end{equation}
and by the optimality condition satisfied by $\hat{w}_B$, together with the characterization of the subdifferential in \eqref{eq:ell1subdiff} we get
\begin{equation}
\hat{u}_{B,k} := [\hat{u}_B]_k = 0 \quad \text{for $k > {N_0}$,}
    \label{eq:ukPN}
\end{equation}
whence $P_{N_0}^\perp \hat{u}_B = 0$, or $P_{N_0} \hat{u}_B = \hat{u}_B$, independently on $B$.

\textit{Step $2$.}
We show that 
\begin{equation}
    \| P_{N_0}^\perp (v- \hat{u}_B) \|_{\ell^2}^2 \leq 2 c_{\rm UB}(M) \left( J_B(v) - J_B(\hat{u}_B) - \frac{1}{2} \| \Se^{-1/2} AB(v-\hat{u}_B)\|_Y^2 \right),
    \label{eq:step2}
\end{equation}
where $c_{\rm UB}(M) = c_{\rm UB}(M; A, \Se, Q, |{\mathcal B}|)$ is given in Proposition \ref{prop:sublevels}. Let us denote 
\[
r_B(v) = J_B(v) - J_B(\hat{u}_B) - \frac{1}{2}\| \Se^{-1/2} A B (v-\hat{u}_B)\|_Y^2.
\]
By direct computations, we have
\[
\begin{aligned}
r_B(v) &= \Phi(v) - \Phi(\hat{u}_B) + \frac{1}{2}\| \Se^{-1/2}AB v \|_Y^2 - \frac{1}{2}\| \Se^{-1/2}AB\hat{u}_B\|_Y^2 - \frac{1}{2}\| \Se^{-1/2}AB(v-\hat{u}_B)\|_Y^2 \\& \qquad - \langle  y,  \Se^{-1} AB (v - \hat{u}_B) \rY \\
&= \Phi(v) - \Phi(\hat{u}_B) + \lY \Se^{-1/2} AB \hat{u}_B,  \Se^{-1/2} AB (v-\hat{u}_B) \rY - \langle B^*( \Se^{-1} A)^* y, v - \hat{u}_B \rangle_{\ell^2} \\ 
&= \Phi(v) - \Phi(\hat{u}_B) + \<  B^*(\Se^{-1} A)^* ( AB \hat{u}_B - y), v-\hat{u}_B \>_{\ell^2}. 
\end{aligned}
\]
Using the definition of $\hat{w}_B$ in \eqref{eq:OC}, we obtain the expression
\[
r_B(v) = \Phi(v) - \Phi(\hat{u}_B) - \langle \hat{w}_B, v-\hat{u}_B \rangle_{\ell^2} =  \sum_{k \in \N} \left( |v_k| - |\hat{u}_{B,k}| - \hat{w}_{B,k}(v_k - \hat{u}_{B,k}) \right).
\]
Moreover, thanks to \eqref{eq:ell1subdiff}, it is easy to verify that each term of the previous summation is non-negative. Consider now $N_0$ introduced in Step 1: by \eqref{eq:wkPN} and \eqref{eq:ukPN},
\[
\begin{aligned}
r_B(v) &\geq \sum_{k > {N_0}} \left( |v_k| - |\hat{u}_{B,k}| - \hat{w}_{B,k}(v_k - \hat{u}_{B,k}) \right) = \sum_{k > {N_0}} \left(|v_k| - \hat{w}_{B,k} v_k  \right) \\
&\geq \sum_{k > {N_0}} \left(  |v_k| - \frac{1}{2}| v_k|  \right) = \frac{1}{2} \| P_{N_0}^\perp v \|_{\ell^1} \geq \frac{1}{2} \| P_{N_0}^\perp v \|_{\ell^2} = \frac{1}{2} \| P_{N_0}^\perp ( v- \hat{u}_B) \|_{\ell^2}.
\end{aligned}
\]
In conclusion,
\[
\| P_{N_0}^\perp ( v- \hat{u}_B) \|_{\ell^2}^2 \leq 2 r_B(v) \| P_{N_0}^\perp (v-\hat{u}_B) \|_{\ell^2} \leq 2 r_B(v) \|  v \|_{\ell^2}.
\]
Since $J_B(v) \leq M$,  by Proposition~\ref{prop:sublevels} we have $\| v\|_{\ell^2} \leq c_{\rm UB}(M)$, hence \eqref{eq:step2} holds.

\textit{Step $3$.} Consider the projection operators associated with the choice $N=N_0$ as in Step 1:
\begin{equation}
\| v - \hat{u}_B \|_{\ell^2}^2 = \| P_{N_0} (v-\hat{u}_B) \|_{\ell^2}^2 +  \| P_{N_0}^\perp (v-\hat{u}_B) \|_{\ell^2}^2.
    \label{eq:start}
\end{equation}

We can bound the first term on the right-hand side of \eqref{eq:start} by means of Lemma \ref{lem:PN2} and get, for any $B \in \mathcal{B}$, that
\[
\begin{aligned}
\| P_{N_0} (v - \hat{u}_B)\|_{\ell^2}^2 &\leq \frac{1}{c_{\rm LB}} \|\Se^{-1/2} ABP_{N_0}(v-\hat{u}_B) \|_Y^2 \\
& \leq \frac{2}{c_{\rm LB}} \|\Se^{-1/2} AB (v-\hat{u}_B) \|_Y^2 + \frac{2L_{F'}}{c_{\rm LB}} \| P_{N_0}^\perp (v-\hat{u}_B) \|_{\ell^2}^2,
\end{aligned}
\]
where the constant $c_{\rm LB}$ depends on $N_0$, which is nevertheless independent of $B$. Collecting now the terms and observing that \eqref{eq:step2} trivially holds  with the larger constant $$C = \max\left\{2 c_{\rm UB}(M), \frac{4}{c_{\rm LB}}\left( 1 + \frac{2L_{F'}}{c_{\rm LB}} \right)^{-1}\right\},$$ by \eqref{eq:start} we obtain
\[
\begin{aligned}
&\| v - \hat{u}_B \|_{\ell^2}^2  \leq \frac{2\|\Se^{-1/2} AB (v-\hat{u}_B) \|_Y^2}{c_{\rm LB}}  + \left( 1 + \frac{2L_{F'}}{c_{\rm LB}} \right) \| P_{N_0}^\perp (v-\hat{u}_B) \|_{\ell^2}^2 \\
&\leq \frac{2\|\Se^{-1/2} AB (v-\hat{u}_B) \|_Y^2}{c_{\rm LB}} + C \left( 1 + \frac{2L_{F'}}{c_{\rm LB}} \right) \left( J_B(v) - J_B(\hat{u}_B) - \frac{1}{2} \| \Se^{-1/2} AB(v-\hat{u}_B)\|_Y^2 \right).
\end{aligned}
\]
Since $\frac{2}{c_{\rm LB}}-\frac{C}{2} \left( 1 + \frac{2L_{F'}}{c_{\rm LB}} \right) \leq 0$ we have that \eqref{eq:condition} holds with 
\begin{equation}
    \label{eq:tilde_cst}
    \tilde c_{\rm ST} = C \left( 1 + \frac{2L_{F'}}{c_{\rm LB}} \right)
    = \max\left\{2 c_{\rm UB}(M)\left( 1 + \frac{2L_{F'}}{c_{\rm LB}} \right), \frac 4{c_{\rm LB}} \right\}.
\end{equation} 
This completes the proof.
\end{proof}

Now the uniqueness of the minimizer $\hat{u}_B$ follows directly from Proposition~\ref{prop:condition}.

\begin{proof}[Proof of Theorem~\ref{thm:xhat_is_unique} (uniqueness)]
Let $\hat{u}$ and $\hat{u}'$ be two minimizers of $J_B$. By Proposition~\ref{prop:condition}, choosing $v = \hat{u}'$ and $M = 0$ (indeed, $J_B(\hat{u}) = J_B(\hat{u}') \leq J_B(0) = 0$) we would have
\[
\| \hat{u} - \hat{u}' \|_{\ell^2}^{2} \leq \tilde c_{\rm ST} (J_B(\hat{u})-J_B(\hat{u}')) = 0.
\]
\end{proof}

Finally, we obtain a stability estimate for the perturbation of $\hat{u}_B$ with respect to modifications of $B$ within the class $\mathcal{B}$.

\begin{proof}[Proof of Theorem \ref{thm:stab}]
Notice that, thanks to the uniform bound on the minimizers \eqref{eq:boundmin} it holds that, independently of $B_1,B_2 \in \mathcal{B}$,
\begin{eqnarray}
\label{eq:uniform_bound_for_minimizers}
J_{B_2}(\hat{u}_1) & = & \frac{1}{2} \| \Se^{-1/2} A B_2 \hat{u}_1 \|_Y^2 - \langle y, \Se^{-1}A B_2\hat{u}_1\rY + \| \hat{u}_1\|_{\ell^1} \nonumber \\
& \leq & \frac{1}{2}L_{F'} c_{\rm UB} + Q \| \Se^{-1}A\||\mathcal{B}| c_{\rm UB} + c_{\rm UB}.
\end{eqnarray}
Then, we can apply Proposition \ref{prop:condition} to $J_{B_2}$ for the choice $v = \hat{u}_1$ with a $M$-level set specified by the upper bound in \eqref{eq:uniform_bound_for_minimizers}. We obtain
\[
\| \hat{u}_1 - \hat{u}_2 \|_{\ell^2}^2 \leq \tilde c_{\rm ST}(J_{B_2}(\hat{u}_1)-J_{B_2}(\hat{u}_2)),
\]
where the explicit expression of the constant $\tilde c_{\rm ST}$ is given in \eqref{eq:tilde_cst}.
Now, since $J_{B_1}(\hat{u}_1) \leq J_{B_1}(\hat{u}_2)$, we have
\[
\begin{aligned}
\| \hat{u}_1 - \hat{u}_2 \|_{\ell^2}^2 &\leq \tilde c_{\rm ST} \big( J_{B_2}(\hat{u}_1)-J_{B_2}(\hat{u}_2) + J_{B_1}(\hat{u}_2)-J_{B_1}(\hat{u}_1) \big).
\end{aligned}
\]

Consider now the terms $J_{B_2}(\hat{u}_1) - J_{B_1}(\hat{u}_1)$: by direct computations, we get 
\[
\begin{aligned}
J_{B_2}(\hat{u}_1) - J_{B_1}(\hat{u}_1) & = \| \Se^{-1/2}AB_2 \hat{u}_1 \|_{Y}^2 - \| \Se^{-1/2}AB_1 \hat{u}_1 \|_{Y}^2  - \lK y, \Se^{-1} A(B_2 - B_1)\hat{u}_1\rY \\
& = \lY \Se^{-1}A(B_2 + B_1) \hat{u}_1 , A(B_2-B_1) \hat{u}_1 \rY - \lX \big(\Se^{-1}A\big)^*y,(B_2 - B_1)\hat{u}_1\rX.
\end{aligned}
\]
Analogously, it holds
\[
\begin{aligned}
J_{B_1}(\hat{u}_2) - J_{B_2}(\hat{u}_2) = \lY \Se^{-1}A(B_1 + B_2) \hat{u}_2 , A(B_1-B_2) \hat{u}_2 \rY - \lX \big(\Se^{-1}A\big)^*y,(B_1 - B_2)\hat{u}_2\rX.
\end{aligned}
\]
Summing the two expressions, we obtain
\[
\begin{aligned}
\|\hat{u}_1 - \hat{u}_2 \|_{\ell^2}^2  \!\!\leq & \tilde c_{\rm ST} \big( \lY \Se^{-1}A(B_1 + B_2) \hat{u}_1 , A(B_2-B_1) \hat{u}_1 \rY 
\! - \! \lY \Se^{-1}A(B_1 + B_2) \hat{u}_2 , A(B_2-B_1) \hat{u}_2 \rY \\
& \qquad - \lX \big(\Se^{-1}A\big)^*y,(B_2 - B_1)\hat{u}_1\rX
+ \lX \big(\Se^{-1}A\big)^*y,(B_2 - B_1)\hat{u}_2\rX \big) \\
= & \ \tilde c_{\rm ST} \big(
\lY \Se^{-1}A(B_1 + B_2) \hat{u}_1 , A(B_2-B_1) (\hat{u}_1 -\hat{u}_2) \rY \\
& \qquad + \lY \Se^{-1}A(B_1 + B_2) (\hat{u}_1 - \hat{u}_2) , A(B_2-B_1) \hat{u}_2 \rY \\
& \qquad - \lX \big(\Se^{-1}A\big)^*y,(B_2 - B_1)(\hat{u}_1-\hat{u}_2)\rX
\big) \\
\leq & \ \tilde c_{\rm ST} \big(
4 \| \Se^{-1}A\| |\mathcal{B}| c_{\rm UB} \| A \| + \|\Se^{-1}A\| Q
\big) \|B_1 - B_2\| \|\hat{u}_1-\hat{u}_2\|_{\ell^2}.
\end{aligned}
\]
In conclusion, letting $C = \| \Se^{-1}A\|\big(
4  |\mathcal{B}| c_{\rm UB} \| A \| + Q
\big)$, we have
\[
\| \hat{u}_1 - \hat{u}_2 \|_{\ell^2} \leq  C \tilde c_{\rm ST} \| B_1 - B_2 \|,
\]
and, since $\hat{x}_i = B_i \hat{u}_i$, it follows that
\[
\begin{aligned}
    \| \hat{x}_1 - \hat{x}_2 \|_{X} &\leq \| B_1 - B_2 \| \|\hat{u}_1\|_{\ell^2} + \| B_2 \| \| \hat{u}_1 - \hat{u}_2 \|_{\ell^2} \\
    & \leq ( c_{\rm UB} + |\mathcal{B}| C \tilde c_{\rm ST}) \| B_1 - B_2\|
\end{aligned}
\]
which concludes the proof with the choice
$c_{\rm ST} = c_{\rm UB} + |\mathcal{B}| C \tilde c_{\rm ST}$.
\end{proof}

\section{Statistical learning framework}
\label{sec:stat}

As shown in the previous section, for every choice of $B\in \mathcal{B}$ and any noisy output $y\in Y$, we can associate a solution $x=R_B(y)\in X$  of the inverse problem $Ax=y$, which depends on $B$. As discussed in Section~\ref{sec:intro} (see in particular Assumption~\ref{ass:stat}), the output $y$ and hence the solution $R_B(y)$ are random variables, and the optimal $B^\star$ is defined as the minimizer of the expected risk $L$ defined by \eqref{eq:loss}. Since $L$ depends on the unknown distribution of the pair $(x,y)$, $B^\star $ is estimated by its empirical version $\widehat{B}$, given by~\eqref{eq:empiricaltarget}. In this section, we provide a finite sample bound on the discrepancy
\[ L(\widehat{B})-L(B^\star),\]
under the assumption that both $x$ and $\eps$ are bounded random variables, and $\mathcal B$ is compact. 
In fact, boundedness is assumed for convenience; for similar techniques with sub-Gaussian random variables, see e.g.\ \cite{ratti2023learned}.
\begin{assumption}
\label{ass:bdd_rv}
There exists $Q_0>0$ such that $\norm{x}_X\leq Q_0$ and $\norm{\varepsilon}_Y \leq Q_0$ almost surely.
\end{assumption}
Assumption \ref{ass:bdd_rv} directly implies a bound $\norm{y}_Y \leq Q = (\norm{A} +1)Q_0$, which was used in Section~\ref{sec:optimization}.

The following result shows the existence of $B^\star$ and $\widehat{B}$, and provides a finite sample bound on $L(\BS) - L(\BOp)$ in probability. It is proved analogously to Proposition~4 in \cite{cucker2002mathematical} (see also \cite[Lemma A.5]{alberti2021learning}).
We recall that, for any $r>0$,  $\mathcal{N}(\mathcal{B},r)$ denotes the \textit{covering number} of $\mathcal{B}$, i.e., the minimum number of balls of radius $r$  (in the strong operator norm on $\LO(\ell^2,X)$) whose union contains $\mathcal{B}$.

\begin{thm}\label{thm:CuckSmale}
Let Assumptions \ref{ass:stat}, \ref{ass:compatibility}, \ref{ass:FBI}, \ref{ass:mathcalB}, and \ref{ass:bdd_rv} hold. There exist a minimizer $\BOp$ of $L$ and, with probability $1$, a minimizer $\BS$ of $\wh{L}$ over $\mathcal{B}$. 
Furthermore,  for all $\eta > 0$,
\begin{equation}
    \Prob_{\vzv \sim \rho^m} \left[ L(\BS) - L(\BOp) \leq \eta \right] \geq 1 - 2\ \mathcal{N}(\mathcal{B},C\eta) e^{-C'm\eta^2}
\label{eq:CuckSmale}
\end{equation}
for some $C,C'>0$ depending only on $A$, $\Se$, $Q_0$ and $\mathcal B$.
\end{thm}
In the above statement, the quantity $\BS$ depends on $\vzv$ and is therefore a random variable. For ease and convenience, in the rest of the section, we denote by $\Prob$ and $\E$ the probability and the expected value computed with respect to the sample $\vzv$, respectively.
\begin{proof}
Consider two different elements $B_1,B_2 \in \mathcal{B}$. Then, $\rho$-a.e.\ in $X\times Y$
\[
\begin{aligned}
\left| \| R_{B_1}(y) - x\|_X^2 - \| R_{B_2}(y) - x\|_X^2 \right| & = \left| \langle R_{B_1}(y) - R_{B_2}(y), R_{B_1}(y)-x + R_{B_2}(y)- x \rX \right| \\
 & \leq \|  R_{B_1}(y) - R_{B_2}(y) \|_X ( \|R_{B_1}(y) \|_X + \| R_{B_2}(y)\|_X + 2 Q_0 ) \\
 & \leq 2 (|\mathcal{B}|c_{\rm UB} + Q_0) c_{\rm ST} \| B_1 - B_2 \|,
\end{aligned}
\]
here we have used Theorem~\ref{thm:stab} and \eqref{eq:boundmin}.
By integrating with respect to the probability distribution $\rho$, the above bound holds for  $L$. Indeed,
\begin{eqnarray}\label{eq:lip1}
 |L(B_1)-L(B_2)| 
 &=& \left|\E[\|R_{B_1}(y) - x\|_X^2] - \E[\|R_{B_2}(y) - x\|_X^2]\right| \nonumber \\
 & \leq &\E\left[|\|R_{B_1}(y) - x\|_X^2 - \|R_{B_2}(y) - x\|_X^2 |\right] \nonumber \\
 & \leq & 2(|\mathcal{B}|c_{\rm UB} + Q_0) c_{\rm ST} \| B_1-B_2\|,
 \end{eqnarray}
 and, by replacing $\rho$ with its empirical counterpart, with probability $1$,
\begin{equation}\label{eq:lip2}
 \begin{split}
 |\wh{L}(B_1)-\wh{L}(B_2)| &=  \left| \frac{1}{m} \sum_{j = 1}^m \|R_{B_1}(y_j) - x_j\|_X^2 -  \frac{1}{m} \sum_{j = 1}^m \|R_{B_2}(y_j) - x_j\|_X^2 \right| \\
 & \leq \frac{1}{m} \sum_{j = 1}^m  \left| \|R_{B_1}(y_j) - x_j\|_X^2 - \|R_{B_2}(y_j) - x_j\|_X^2 \right| \\
 &\leq 2(|\mathcal{B}|c_{\rm UB} + Q_0) c_{\rm ST} \| B_1-B_2\|.
 \end{split}
\end{equation}
Since both $L$ and $\wh{L}$ are  continuous functionals on $\mathcal{B}$ and $\mathcal{B}$ is compact with respect to the  operator topology, the corresponding minimizers $\BOp$ and $\BS$ over $\mathcal{B}$ exist (almost surely for $\BS$). 

Next, we  notice that 
\[
\sup_{B \in \mathcal{B}} |\wh{L}(B) - L(B)| \leq \frac{\eta}{2} \quad \Rightarrow \quad
L(\BS) - \wh{L}(\BS) \leq  \frac{\eta}{2} \quad \text{and} \quad \wh{L}(\BOp) - L(\BOp) \leq \frac{\eta}{2},
\]
which ultimately implies that
\[
0\leq L(\BS) - L(\BOp) =  
\left(L(\BS) - \wh{L}(\BS)\right)+ 
\left(\wh{L}(\BS) - \wh{L}(\BOp)\right) + 
\left(\wh{L}(\BOp) - L(\BOp)\right) \leq \eta,
\]
since the central difference is negative by the definition of $\BS$. Thus,
\[
\Prob \left[ L(\BS) - L(\BOp) \leq \eta \right] \geq \Prob \left[ \sup_{B \in \mathcal{B}} |\wh{L}(B) - L(B)| \leq \frac{\eta}{2} \right].
\]

We now provide a lower bound for the latter term. 
In view of \eqref{eq:lip1} and \eqref{eq:lip2}, by using the reverse triangle inequality, for every $B_1,B_2 \in \mathcal{B}$,
\[
\begin{aligned}
 \left| |\wh{L}(B_1) - L(B_1)| - |\wh{L}(B_2) - L(B_2)|  \right| & \leq |\wh{L}(B_1) - \wh{L}(B_2)| + |L(B_1)-L(B_2)| \\ & \leq 4 (|\mathcal{B}|c_{\rm UB} + Q_0) c_{\rm ST} \| B_1-B_2\|.    
\end{aligned}
\]
Let now $N=\mathcal{N}\left( \mathcal{B}, \frac{\eta}{16(|\mathcal{B}|c_{\rm UB} + Q_0) c_{\rm ST}} \right)$ and consider a discrete set $B_1, \ldots, B_N$ such that the balls $\mathcal{B}_k$ centered at $B_k$ with radius $r = \frac{\eta}{16(|\mathcal{B}|c_{\rm UB} + Q_0) c_{\rm ST}}$ 
cover the entire $\mathcal{B}$. In each ball $\mathcal{B}_k$, for every $B \in \mathcal{B}_k$ it holds
\[
 \left| |\wh{L}(B) - L(B)| - |\wh{L}(B_k) - L(B_k)|  \right| \leq 4(|\mathcal{B}|c_{\rm UB}+Q_0)c_{\rm ST} \| B - B_k \| \leq \frac{\eta}{4}.
\]
Therefore, the event $|\wh{L}(B) - L(B)| > \frac{\eta}{2}$ implies the event $|\wh{L}(B_k) - L(B_k)| > \frac{\eta}{4}$, and a bound (in probability) of this term can be provided by standard concentration results. Indeed, $\wh{L}(B_k)$ is the sample average of $m$ realizations of the random variable $\|R_{B_k}(y) - x\|_X^2$, whose expectation is $L(B_k)$. Moreover, such random variable is bounded by $(|\mathcal{B}|c_{\rm UB} + Q_0)^2$ by assumption, and therefore via Hoeffding's inequality
\[
\begin{aligned}
 \Prob \left[{\sup_{B \in \mathcal{B}_k}} |\wh{L}(B) - L(B)| > \frac{\eta}{2} \right] & \leq \Prob \left[ |\wh{L}(B_k) - L(B_k)| > \frac{\eta}{4} \right] \leq 2 e^{-\frac{m\eta^2}{8(|\mathcal{B}|c_{\rm UB}+Q_0)^4}}.
 \end{aligned}
 \]
 Notice that this inequality holds uniformly in $k$. Finally, we obtain
 \begin{equation*}
 \begin{split}
  \Prob  \left[ \sup_{B \in \mathcal{B}} |\wh{L}(B) - L(B)| \leq \frac{\eta}{2} \right]  &= 1 - \Prob \left[ \sup_{B \in \mathcal{B}} |\wh{L}(B) - L(B)| > \frac{\eta}{2} \right] \\
  & \geq
  1- \sum_{k=1}^N \Prob  \left[ \sup_{B \in \mathcal{B}_k} |\wh{L}(B) - L(B)| > \eta \right] \\
  &\geq 1 - 2N e^{-\frac{m\eta^2}{8(|\mathcal{B}|c_{\rm UB}+Q_0)^4}}.
 \end{split}
 \end{equation*}
This proves \eqref{eq:CuckSmale}, with $C=(16(|\mathcal{B}|c_{\rm UB} + Q_0) c_{\rm ST})^{-1}$ and $C'=(8(|\mathcal{B}|c_{\rm UB}+Q_0)^4)^{-1}$.
\end{proof}

We now provide a more explicit expression for the sample error estimate under the assumption that the covering numbers of the set $\mathcal{B}$ have a specific decay rate. It is possible to obtain similar bounds, for example, whenever the singular values of the compact embedding of $\mathcal{B}$ into its ambient space show a polynomial decay.
\begin{cor} \label{thm:sample_error_prob}
Let Assumptions \ref{ass:stat}, \ref{ass:compatibility}, \ref{ass:FBI}, \ref{ass:mathcalB}, and \ref{ass:bdd_rv} hold. Assume further that
\begin{equation}
\log(\mathcal{N}(\mathcal{B},r)) \leq C r^{-1/s},
    \label{eq:covering}
\end{equation}
holds for some constants $C,s>0$. Suppose that $\tau>0$. Then, for sufficiently large values of $m$, 
    \begin{equation}\label{eq:tail}
        \Prob
        \left[ L(\hat{B})-L(\BOp)\leq \left(\frac{\alpha_1 + \alpha_2\sqrt{\tau}}{\sqrt{m}}\right)^{1-\frac{1}{1+2s}}\right] \geq 1-e^{-\tau}
    \end{equation}
for some  $\alpha_1,\alpha_2>0$ depending only on $A$, $\mathcal{B}$, $\Se$, and $Q_0$.
\end{cor}
\begin{proof}
Let $a_1 = C\bigl(16(|\mathcal{B}| c_{\rm UB}+Q_0)c_{\rm ST}\bigr)^{1/s}$ and $a_2 = (8(|\mathcal{B}|c_{\rm UB}+Q_0)^4)^{-1}$: then, by 
\eqref{eq:CuckSmale} and \eqref{eq:covering}, for $\eta\in (0,1]$ and $m$ sufficiently large (namely, such that $a_2 m\eta^{2+1/s}\geq a_1$) we have
\begin{equation}
    \label{eq:tailbound2}
\Prob
\left[ L(\BS) - L(\BOp) \leq \eta \right]
\geq 1-e^{-\eta^{-1/s}(a_2 m\eta^{2+1/s}-a_1)}
\geq 1-e^{-(a_2 m\eta^{2+1/s}-a_1)}
=1 - e^{-\tau},
\end{equation}
being 
$\tau = a_2 m \eta^{2+1/s} - a_1$.
We can rephrase \eqref{eq:tailbound2} by expressing $\eta$ as a function of $\tau$ and $m$: 
\[
1 - e^{-\tau} \leq \left[ L(\BS) - L(\BOp) \leq \left(\frac{a_1 + \tau}{a_2 m} \right)^{\frac{1}{2+1/s}} \right] \leq \Prob
        \left[ L(\hat{B})-L(\BOp)\leq \left(\frac{\alpha_1 + \alpha_2\sqrt{\tau}}{\sqrt{m}}\right)^{1-\frac{1}{1+2s}}\right],
\]
where $\alpha_1=\sqrt{a_1/a_2}$ and $\alpha_2=\sqrt{1/a_2}$, as $\sqrt{a_1+\tau}\leq \sqrt{a_1} + \sqrt{\tau}$. 
This concludes the proof.
\end{proof}

Note that, besides the bound in probability \eqref{eq:tail}, we can also provide a bound in expectation for the excess risk. Indeed, by \eqref{eq:tailbound2}, 
\[
\Prob 
\left[ L(\BS) - L(\BOp) < \eta \right] \leq \min\big\{1,e^{a_1 \eta^{-1/s} -a_2 m\eta^2}\big\},
\]
and by the tail integral formula (notice that $e^{a_1 \eta^{-1/s} - a_2 m\eta^2} = 1$ for $\eta = \hat{\eta} = k_1 m^{-\frac{s}{2s+1}}$)
\[
\begin{aligned}
    \E
    \big[L(\hat{B})-L(\BOp)\big] &= \int_{0}^\infty \Prob
    \left[ L(\BS) - L(\BOp) < \eta \right]d \eta \\
    &\leq \hat{\eta} +  e^{a_1 \hat{\eta}^{-1/s}} \int_{\hat{\eta}}^\infty e^{-a_2 m \eta^2}d\eta \leq k_1 m^{-\frac{s}{2s+1}} + k_2 m^{-\frac{1}{2}},
\end{aligned}
\]
which allows us to conclude that, for sufficiently large $m$,
\begin{equation}
    \E
    \big[L(\hat{B})-L(\BOp)\big] \lesssim m^{-\frac{s}{2s+1}}. \label{eq:bound_covering}
\end{equation}

The bound~\eqref{eq:bound_covering} should be compared with the one in \cite[Corollary 7.2] {ratti2023learned} setting $\alpha = 1$ and $q=2$, namely
\begin{equation}
    \E
    \big[L(\hat{B})-L(\BOp)\big] \lesssim m^{-\frac{1}{2}}  \label{eq:bound_chaining}
\end{equation}
provided that $s>\frac{1}{2}$, compare with the assumptions of \cite[Proposition 7.3]{ratti2023learned}. 
The rate in  \eqref{eq:bound_covering} is worse than the rate in \eqref{eq:bound_chaining}, obtained via a chaining argument. On the other side, the bound in Corollary~\ref{thm:sample_error_prob} holds in probability, whereas the bound~\eqref{eq:bound_chaining} 
 holds only in expectation.

\section{On the choice of the parameter class \texorpdfstring{$\mathcal{B}$}{B}}
\label{sec:examples}
\subsection{Compact perturbations of a known operator}\label{sec:perturbation_known}
We first consider an example of a class $\mathcal{B}$ consisting of certain compact perturbations of a reference operator $B_0$.
Under rather general assumptions, we can provide an estimate of the covering numbers of $\mathcal{B}$. The proofs of this section are standard, and are postponed to Appendix~\ref{sec:proofs_example}.
 
 Assume that $A$ is injective. We fix an injective operator $B_0 \colon \ell^2 \rightarrow X$ and two compact operators $E_1,E_2\colon \ell^2\to\ell^2$ such that $\| B_0 \|_{\LO(\ell^2,X)}\leq 1$ and $\| E_i \|_{\LO(\ell^2)}\leq 1$ for $i=1,2$ (here, $\LO(X):=\LO(X,X)$). For every finite set $I\subset\N$, let $c_I>0$.
We define 
\begin{equation}
\mathcal{B} = \{ B_0(\operatorname{Id} + K):  K \in \mathcal{H}\}, 
    \label{eq:Bexample}
\end{equation}
where the set of operators $\mathcal{H}$ is defined as 
\begin{multline}
\mathcal{H} = \{ K = E_1 T E_2 :  T \in \LO(\ell^2), \; \| T \|_{\LO(\ell^2)}\leq 1 \; \text{and} \\ \|(\operatorname{Id} + K)u\|_{\ell^2}\ge c_I \|u\|_{\ell^2} \; \text{ for every finite $I\subset\N$ and $u\in \ell^2_I$}\},
    \label{eq:Hexample}
\end{multline}
where
$\ell^2_I := \operatorname{span}\{ e_i:  i \in I\}$.

\begin{lemma}\label{lem:Bcompactexample}
    The set $\mathcal{B}$ defined in \eqref{eq:Bexample}-\eqref{eq:Hexample} satisfies Assumption~\ref{ass:mathcalB}.
\end{lemma}

The theoretical results of Section~\ref{sec:stat} rely on the covering numbers of $\mathcal{B}$. We now derive an estimate of the form \eqref{eq:covering} in the case when $\mathcal{B}$ is given by \eqref{eq:Bexample}.
For the sake of clarity, we indicate the norm  used to perform the covering explicitly.

\begin{proposition}\label{prop:coveringexample}
   Let $\mathcal{B}$ be defined as in \eqref{eq:Bexample}-\eqref{eq:Hexample}. Suppose that $E_1 = E_2 = E$, being $E\in\LO(\ell^2)$ a positive operator whose eigenvalues $\lambda_n$ satisfy
   \[
\lambda_n \le c n^{-s},\qquad n\in\N
   \]
   for some $c,s>0$. Take $s'<s$.
   Then there exists $C>0$ such that
\begin{equation}
\log \mathcal{N}(\mathcal{B},r; \|\cdot \|_{\LO(\ell^2,X)}) \leq C r^{-\frac{2s+1}{2ss'}},\qquad r>0.
\label{eq:covering_example_comp}    
\end{equation}
\end{proposition}

With this estimate on the covering numbers, it is possible to apply Corollary~\ref{thm:sample_error_prob} and obtain a finite sample bound in the case when $\mathcal{B}$ is defined as in \eqref{eq:Bexample}.
It is worth observing that the most relevant approaches in the literature focusing on learning sparsifying dictionaries through the solution of a bilevel problem, such as \cite{horesh-haber-2009,huang-haber-horesh-2012,2022-ghosh-etal,ghosh-etal-2024}, are formulated in a finite-dimensional setup, and so naturally satisfy the assumptions of this example.

\subsection{Learning the mother wavelet}

We consider here the problem of learning the mother wavelet. In other words, the set $\mathcal{B}$ will consist of synthesis operators associated to a family of wavelets.

Consider the Hilbert space $X=L^2(\R^d)$ and, for $\psi\in L^2(\R^d)$, define
\[
\psi_{j,k}(x)=2^{\frac{jd}{2}}\psi(2^j x-k),\qquad \text{for a.e.\ $x\in\R^d$, $j\in\Z$, $k\in\Z^d$.}
\]
For $f\in L^2(\R^d)\cap L^1(\R^d)$ (the space $L^1(\R^d)$ is added just to simplify the exposition, thanks to the continuity of $\hat f$), define
\[
\|f\|_W^2=\sup_{\xi\in\R^d}\sum_{j,k}|\hat f(2^{-j}\xi+2\pi k)|^2,\qquad \hat f(s)=\int_{\R^d} f(x) e^{-ix\cdot s}\,dx,
\]
and consider
\[
W=\{f\in L^2(\R^d)\cap L^1(\R^d):\|f\|_W<+\infty\}.
\]
It is easy to verify that $\|\cdot\|_W$ is a norm on $W$. It is well known that, for $\psi\in W$, the family $\{\psi_{j,k}\}_{j\in\Z,k\in\Z^d}$ is a Bessel sequence of $L^2(\R^d)$, namely, the synthesis operator
\[
B_\psi\colon \ell^2(\Z\times \Z^d)\to L^2(\R^d),\qquad (c_{j,k})\mapsto \sum_{j,k} c_{j,k} \psi_{j,k},
\]
is well defined and bounded (see, e.g., \cite[Section~3.3]{daubechies-1992} and \cite[Theorem~3]{1999-jing}).

We now construct a compact family of ``mother wavelets'' $\psi$ in $W$. (Note that the term mother wavelet here is used even though the family $\psi_{j,k}$ need not be a frame of $L^2(\R^d)$.)
\begin{lemma}\label{lem:wave}
    Let $\Psi\subseteq W$ be a compact set (with respect to $\|\cdot\|_W$) and $a>0$. Set
    \[
\Psi_a = \{\psi\in\Psi: \|B_\psi x\|_{L^2(\R^d)} \geq a \| x \|_{\ell^2} \ \forall x \in \ell^2\}.
    \]
Then
    \[
\mathcal{B}_{\Psi_a} = \{B_\psi:\psi\in\Psi_a\}
    \]
    is a compact subset of $\LO (\ell^2,L^2(\R^d))$ with respect to the operator norm, and
    \[
\mathcal{N}\left( \mathcal{B}_{\Psi_a},\delta; \|\cdot \|_{\LO(\ell^2,L^2(\R^d))} \right) \le \mathcal{N}\left( \Psi,(2\pi)^{\frac32 d}\,\delta; \|\cdot \|_{W} \right),\qquad \delta>0.
    \]
\end{lemma}
\begin{proof}
By the estimates  in the proof of Theorem~3 in \cite{1999-jing}, we have
\[
\sum_{j,k}|\langle f,\psi_{j,k}\rangle_{L^2(\R^d)} |^2 \le (2\pi)^{-3d} \|\psi\|_W^2\|f\|^2_{L^2(\R^d)},\qquad \psi\in W,\,f\in L^2(\R^d). 
\]
In other words, we have the following bound on the norm of the analysis operator:
$
\|B^*_\psi\|\le (2\pi)^{-\frac32 d} \|\psi\|_W.
$
As a consequence, we have
\[
\|B_\psi\|\le (2\pi)^{-\frac32 d} \|\psi\|_W,\qquad \psi\in W.
\]
In other words, the linear map
\[
\zeta\colon W\to \LO(\ell^2,L^2(\R^d)),\qquad 
\psi\mapsto B_\psi 
\]
is linear and bounded, with $\|\zeta\|\le (2\pi)^{-\frac32 d}$.

Since $\zeta$ is bounded and $\Psi$ is compact, we have that $ \{B_\psi:\psi\in\Psi\}=\zeta(\Psi)$ is compact. Thus, $\mathcal{B}_\Psi$ is the intersection of a compact set and a closed set, and so it is compact.

Finally, the estimate on the covering numbers of $\mathcal{B}_{\Psi_a}$ immediately follows:
\[
\mathcal{N}\left( \mathcal{B}_{\Psi_a},\delta \right)\le 
\mathcal{N}\left( \zeta(\Psi),\delta \right)
\le 
\mathcal{N}\left( \Psi,\|\zeta\|^{-1}\delta \right)
\le 
\mathcal{N}\left( \Psi,(2\pi)^{\frac32 d}\,\delta \right).
\]
\end{proof}
We can now combine this result with Corollary~\ref{thm:sample_error_prob} and obtain the following corollary, in which we compare the loss corresponding to the optimal mother wavelet $\psi^*$ with the loss corresponding to the mother wavelet $\widehat\psi$ obtained by minimizing the empirical risk.
\begin{cor}
    Consider the settings of Corollary~\ref{thm:sample_error_prob} and of Lemma~\ref{lem:wave}. Suppose that 
    \[
\log \mathcal{N} (\Psi,r; \|\cdot \|_{W}) \le C r^{-1/s},\qquad r>0,
    \]
    for some $C,s>0$. There exist $\alpha_1,\alpha_2>0$ such that the following is true. Take $\tau>0$ and
 \[
 \psi^* \in\argmin_{\psi\in\Psi_a} L(B_\psi),\qquad \widehat\psi \in\argmin_{\psi\in\Psi_a} \widehat L(B_\psi).
 \]   
Then, with probability larger than $1-e^{-\tau}$, we have
\[
L(B_{\widehat \psi})-L(B_{\psi^*})\leq \left(\frac{\alpha_1 + \alpha_2\sqrt{\tau}}{\sqrt{m}}\right)^{1-\frac{1}{1+2s}}.
\]
\end{cor}
\begin{proof}
    By Lemma~\ref{lem:wave}, we have
    \[
\log \mathcal{N}\left( \mathcal{B}_{\Psi_a},r \right) \le \log \mathcal{N}\left( \Psi,(2\pi)^{\frac32 d}\,r \right) \le C (2\pi)^{-\frac{3d}{2s}} r^{-1/s}.
    \]
    The result is now a direct consequence of Corollary~\ref{thm:sample_error_prob}.
\end{proof}

We remark that the task of finding an optimally sparsifying wavelet transform by learning the mother wavelet has been recently studied in relation to applications both on 1D and 2D signals (see \cite{recoskie2018learning,recoskie2018learning2,frusque2022learnable}).

\section{Numerical implementation and experiments}
\label{sec:numerics}

In this section, we discuss numerical strategies through which the bilevel problem defined by \eqref{eq:bilevel} can be approximately solved, and present some related experiments aiming at validating the theoretical findings of the paper, showcasing a competitive comparison with alternative techniques (such as dictionary learning), and illustrating applications to challenging ill-posed problems.

Throughout this section, we assume that the spaces $X$ and $Y$ are finite-dimensional (and, in particular, we let $X = \R^n$). As a consequence, since by Assumption \ref{ass:stat} the covariance $\Se$ is injective,  the term $\displaystyle \frac{1}{2}\| \Se^{-1/2}y\|_Y^2$ in \eqref{eq:xhat} (which is irrelevant for optimization purposes) can be added to \eqref{eq:xhat}, obtaining the following simplified expression of the inner minimization problem: 
\begin{equation}
R_B(y) = B \hat{u}_B = B \argmin_{u \in \R^n} \Big\{ \frac{1}{2} \| \Se^{-1/2}(A\Bs u-y) \|_Y^2 + \| u \|_{1} \Big\}.
\label{eq:easy_xhat}    
\end{equation}
When $\Se = \sigma^2 I$, this reduces to the familiar form of Lasso regression with $\ell^1$ regularization parameter $1/\sigma^2$. Problem~\eqref{eq:easy_xhat} makes clear that the regularized solution $\hat x_B=R_B(y)$ of Problem~\eqref{eq:xhat} is sparse with respect to the basis associated with the invertible matrix $B$.

 In particular, the expression in \eqref{eq:easy_xhat} is known as the \textit{synthesis} formulation of such a problem, as opposed to the \textit{analysis} one in which the transform operator appears within the 1 norm. 

\subsection{On the numerical resolution of the bilevel problem} \label{subsec:algo}

We now consider the bilevel problem of minimizing the empirical loss \eqref{eq:risk} for a fixed dataset $\{(x_j,y_j)\}_{j=1}^m$ subject to the inner problem \eqref{eq:easy_xhat}. Such a problem is extremely challenging, since the inner minimization problem does not have an explicit formula for the solution, and the first-order optimality conditions associated with it are non-differentiable with respect to the parameter $B$. Among the most effective strategies developed for the solution of this bilevel problem, \cite{horesh-haber-2009} proposed a local sensitivity analysis of the inner problem's solution, leading to an effective algorithm, of which we consider a simple but fruitful relaxation. On the other side, in our numerical experiments, more recent methods based on automatic differentiation and neural networks proved to be unsuccessful for the general formulation of the problem (leading to hard-to-train unrolled or deep equilibrium architectures) and show good results only for the denoising problem, as we discuss in detail below. For this reason, in this subsection, we first propose a specific strategy suited for the denoising problem, and then discuss an approximate method to tackle the general formulation of the problem. 

\paragraph{An exact formula for a denoising problem}
Let us consider the case in which $X = Y = \R^n$, the noise covariance is simply $\Se = \sigma^2 I$, and the forward operator is $A = I$, so that the inverse problem reduces to a denoising task. Assume moreover that the minimization of the empirical loss in \eqref{eq:empiricaltarget} is performed on the space $\mathcal{B} = \{B = \alpha Q:\ 0< \underline{\alpha}\leq \alpha \leq \overline{\alpha},\ Q \in \R^{n\times n}  \text{ orthogonal}\}$ for fixed $0<\underline{\alpha}\leq\overline{\alpha}$. In this case, the inner problem \eqref{eq:easy_xhat} admits the following explicit solution:
\begin{equation}
R_B(y) = R_{\alpha Q}(y) = Q S_{\frac{\sigma^2}{\alpha}}(Q^T y),
\label{eq:inner_denoise}    
\end{equation}
where $S_\beta(u)$ denotes the (component-wise) soft-thresholding operator  $[S_{\beta}(u)]_k = \operatorname{sign}(u_k)(|u_k|-\beta)^+$. Although this expression is non-differentiable with respect to $\alpha$ and $Q$, it can be successfully treated via subgradient-based schemes.
As a result, we can employ automatic differentiation tools to iteratively update the parameters $\alpha$ and $Q$ in \eqref{eq:inner_denoise} so as to minimize the following loss:
\[
\mathcal{L}_1(\alpha Q) = \widehat{L}(\alpha Q) + \frac{\mu_{\rm reg}}{2}(\| Q^T Q - I\|_{\text{Fro}}^2 + \| Q Q^T - I\|_{\text{Fro}}^2 ),
\]
where $\|\cdot\|_{\text{Fro}}$ is the Frobenius norm and the regularization term is designed to promote the orthogonality of the matrix $Q$, instead of employing a projection on the Stiefel manifold at each iteration.

\paragraph{An approximate solver based on $\ell^1$-norm relaxation}
We now consider a strategy to treat the bilevel problem for general forward operators, which may also be applied when the matrix A is not square.
One way to overcome the non-differentiability of the outer loss with respect to $B$ is to consider a slightly modified version of the inner problem, namely:
\begin{equation}
R_B^\nu(y) = B \hat{u}^\nu_B(y) = B \argmin_{u \in \R^n} \Big\{ \frac{1}{2} \| \Se^{-1/2}(A\Bs u-y) \|_Y^2 + \| u \|_{1,\nu} \Big\},
\label{eq:easy_xhatnu}
\end{equation}
being $\| u \|_{1,\nu} = \sum_{i=1}^n \sqrt{u_i^2+\nu^2}$ and $\nu>0$ a small parameter such that the solution of Problem~\eqref{eq:easy_xhatnu} is close to the exact solution $R_B(y)$. For the ease of notation, we assume here that $\Se=\sigma^2 I$. Thanks to the proposed relaxation, it is possible to characterize $R_B^\nu(y)$ by the first-order optimality condition of \eqref{eq:easy_xhatnu}, namely
\[
B^TA^T(AB\hat{u}^\nu_B(y) - y) + \sigma^2 w_\nu(\hat{u}^\nu_B(y)) = 0,
\]
being $[w_\nu(u)]_k = \frac{u_k}{\sqrt{u_k^2 + \nu^2}}$. Finally, by differentiating such an expression, we can compute the first-order variation of $\hat{u}^\nu_B$ with respect to $B$ along a direction $\Theta \in \R^{n \times n}$ as follows:
\[
(\hat{u}^\nu_B)'[\Theta] = (B^T A^TA B + \sigma^2 W_\nu(\hat{u}^\nu_B))^{-1}(\Theta^TA^T(y - AB \hat{u}^\nu_B) - B^T A^TA\Theta \hat{u}^\nu_B) 
\]
with $W_\nu(u) = \operatorname{diag}\Big(\frac{\nu^2}{(u_k^2 + \nu^2)^{3/2}}\Big)$. This allows us to compute the gradient of the following regularized loss:
\[
\mathcal{L}_2(B) = \widehat{L}(B) + \frac{\mu_{\text{reg}}}{2} \| B \|_{\text{Fro}}^2\ ,
\]

\[
\nabla\mathcal{L}_2(B) = \frac{2}{m} \sum_{j=1}^m \Big((B\hat{u}^\nu_B(y_j) - x_j -A^TAB\beta_j)\hat{u}^\nu_B(y_j)^T - (A^T(AB\hat{u}^\nu_B(y_j) -y_j))\beta_j^T \Big)+ \mu_{\rm reg} B
\]
being $\beta_j = (B^T A^TA B + \sigma^2 W_\nu(\hat{u}^\nu_B(y_j)))^{-1} B^T(B\hat{u}^\nu_B(y_j) - x_j)$, and can be employed by any first-order scheme for the minimization of $\mathcal{L}_2$. 
In this case, regularization is not needed to enforce desirable properties of the matrix $B$, such as orthogonality.
Instead, we include the quadratic term $\frac{\mu_{\rm reg}}{2} \| B \|_{\rm Fro}^2\ $ to guarantee the stability of the optimization scheme and reduce the risk of overfitting.

\subsection{Numerical experiments}

We present three numerical experiments on synthetic datasets, each with a specific goal. First, a 1D denoising problem is used to numerically validate the theoretical sample error decay from Section \ref{sec:stat}. Second, we compare our supervised strategy against dictionary learning on a 2D denoising task. Third, we demonstrate the effectiveness of the relaxation-based algorithm from Section \ref{subsec:algo} for 1D deblurring.

\paragraph{Example 1: a 1D-denoising problem}

For this experiment, we construct datasets of 1D discrete signals $x_j \in\R^n$ with $n=100$, such that each of them has only $s =\{8,9,10\}$ (randomly selected) non-zero entries with values in the range $[-1,-0.5]\cup[0.5,1]$, so that each signal $x_j$ is sparse with respect to the canonical basis of $\R^n$. The observed vectors are $y_i = x_j + \epsilon_j$, where $\epsilon_j \sim \mathcal{N}(0,\sigma^2 I)$ and $\sigma = 0.25$, resulting in a rather low average value of the PSNR between $y_j$ and $x_j$ (approx.\ $18.10$). 

In this simplified scenario, the ground truth signals $\{x_j\}$ are sparse with respect to the canonical basis by construction. This provides strong prior knowledge, allowing us to assume that an optimal synthesis operator $B^\star$ (a minimizer of the expected loss $L$) should belong to the class of scaled identity matrices, $B = \alpha I$. This ansatz simplifies the problem to finding the optimal regularization parameter for a standard Lasso denoiser. We can therefore establish a strong benchmark by finding the optimal operator within this simplified class, $B^\star = \alpha^\star I$. We identify the optimal $\alpha^\star$ by minimizing the expected loss $L(\alpha I)$, which we approximate by performing a line search on the empirical loss $\widehat{L}(\alpha I)$ computed over a very large, separate dataset (distinct from the training sets used to find $\widehat{B}$). This $B^\star$ serves as a ``near-optimal'' benchmark to measure the excess risk $L(\widehat{B}) - L(B^\star)$.

In order to visualize the decay of the sample error as the size of the training set grows, we first select the sizes $m \in \{10^3, 2\cdot10^3,6\cdot 10^3,10^4,2\cdot 10^4,6\cdot 10^4,10^5\}$ and, for each size, we construct $K=10$ different training sets. We then minimize the loss $\mathcal{L}_1$ associated to each set as discussed in Section \ref{subsec:algo}, initializing the algorithm with a random orthogonal matrix and performing a maximum of \texttt{n\_epochs} = 10000 iterations of Adam with a stepsize \texttt{learning\_rate} = $10^{-3}$ and choosing $\mu_{\rm reg}=10^{-5}$. We obtain a different $\widehat{B}$ for each experiment, and compute the sample error $L(\widehat{B})-L(B^\star)$, approximating the expected loss with the empirical average on a large test set. Finally, for each value of $m$ we average the sample error across the $K$ repeated experiments, and illustrate the dependence of the resulting quantity on $m$ in Figure \ref{fig:1}.
\begin{figure}
    \centering
    \includegraphics[width=0.5\linewidth]{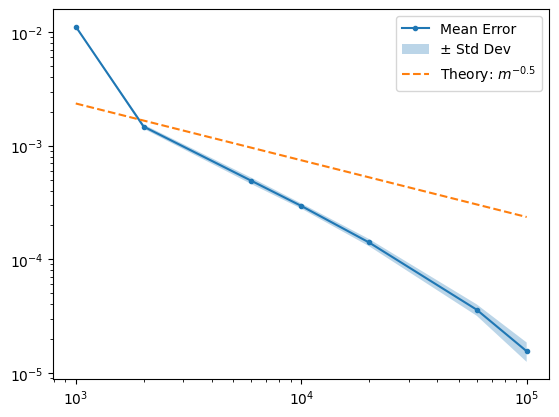}
    \caption{Decay of the average sample error for increasing sample size.}
    \label{fig:1}
\end{figure}

As shown in Figure~\ref{fig:1}, as the sample size increases, the sample error (blue line) decreases faster than the theoretical bound \eqref{eq:tail} (the orange line when $s\to\infty$). This behavior demonstrates the effectiveness of the approach proposed in this work.  The improvement of the decay rate in the numerical experiments might be the result of many factors, including the finite-dimensional (and simple) nature of the considered problem, or some specific advantages induced by the use of the quadratic MSE loss. This result opens the way for further theoretical investigations, possibly in the spirit of \cite[Theorem C*]{cucker2002mathematical}.

\paragraph{Example 2: a 2D-denoising problem}
In this experiment, we consider the task of denoising images $y_j = x_j + \eps_j$ synthetically generated by adding independent noise realizations $\eps_j \sim \mathcal{N}(0,\sigma^2 I)$ (with $\sigma = 0.1$) to images of size $128 \times 128$ containing ellipses with random positions, aspect ratios, orientations and intensities. 
To reduce the computational burden of learning an operator $B \in \R^{n \times n}$ with $n=128^2$, we consider a patch-based denoiser. Namely, from each image we extract non-overlapping patches of size $p \times p$ (with $p=8$), and apply our supervised strategy to minimize $\mathcal{L}_1$ on the resulting dataset of patches, thus learning an operator $B \in \R^{p^2 \times p^2}$. We consider a training dataset of $m=1000$ images (i.e., $256000$ patches), computing a maximum of \texttt{n\_epochs} = 20000 iterations of Adam with a stepsize \texttt{learning\_rate} = $10^{-3}$ and choosing $\mu_{\rm reg}=10^{-3}$. The denoised patches are then recollected to form the denoised images. In Figure \ref{fig:2} we report the learned operator $Q$, being $B = \alpha Q$. To ease the visualization, we represent each column of $Q$ (i.e., each element of the learned basis) not as a vector in $\R^{p^2}$ but as a patch in $\R^{p \times p}$. The learned basis resembles a wavelet one, reflecting the cartoon-like nature of the ellipses' images.
\begin{figure}
    \centering
\includegraphics[width=0.95\linewidth,trim={3 3 2 3}, clip]{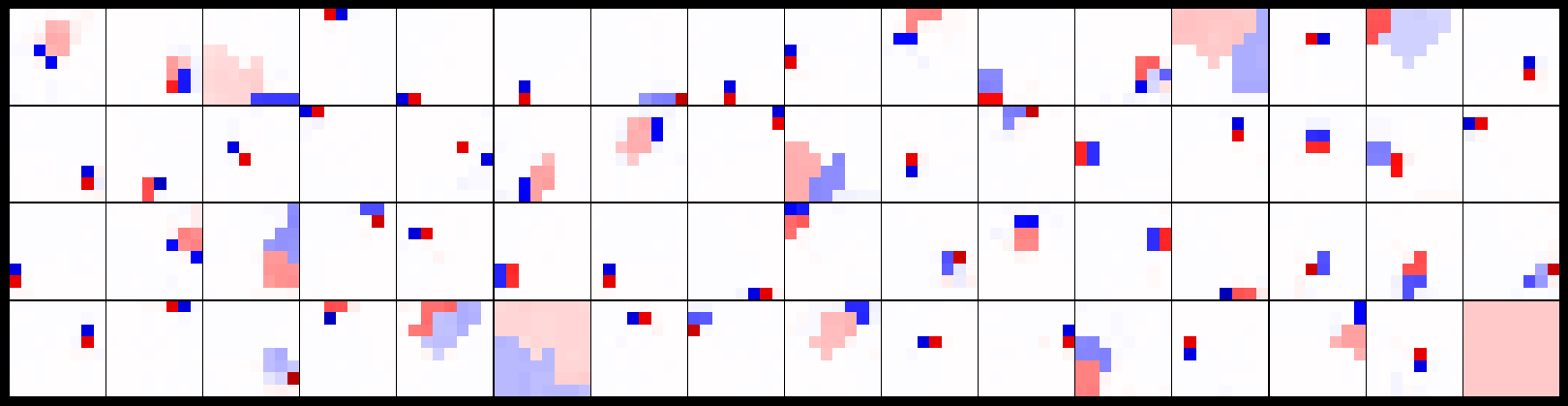}
    \caption{Learned sparsifying transform. Each column of $B$ is represented as an $8 \times 8$ patch.}
    \label{fig:2}
\end{figure}

Since in this case we do not have an insight into an optimal $B^\star$ to use as a benchmark, we compare our strategy with Dictionary Learning (see Appendix \ref{app:DL} for a theoretical comparison). In particular, employing the online algorithm developed in \cite{mairal2009online}, we construct a sparsifying dictionary $D_{\rm DL}$ leveraging the dataset of ground truth patches extracted from $\{x_j\}_{i=1}^{m}$ and later select $B = \alpha D_{\rm DL}$.  Here, $\alpha$ controls the intensity of the regularization depending on the noise level and is selected by a line-search technique, minimizing the empirical risk of the denoising task. Compared with our strategy, this Dictionary-Learning-based technique requires to specify two parameters: in addition to $\alpha$, indeed, we must also select the internal regularization parameter $\mu_{\rm DL}$ of the online algorithm of \cite{mairal2009online}
(which controls the level of desired sparsity induced by the dictionary $D_{\rm DL}$). Tuning $\mu_{\rm DL}$, which can be done again via line-search, is computationally expensive and not straightforward, as the Dictionary Learning stage is not aware of the physics of the inverse problems we aim at solving, but only of the properties of the ground-truth images. Instead, our technique does not require tuning any parameters, learns simultaneously the basis and the regularization parameter, and achieves slightly better performance than the one associated with the best choice of parameters in Dictionary Learning (in this example, $\mu_{\rm DL} = 0.06$ and $\alpha = 0.1537$). The results are reported in Table~\ref{tab:1} and in Figure~\ref{fig:3}.
\begin{table}[h!]
    \centering
\begin{tabular}{l|c|c|c}
 & Noisy data & Dictionary Learning & Our Method \\ \hline
MSE & 1.0001E-02 & 1.78891E-03 & \textbf{1.6403E-03} \\ 
\hline
\end{tabular}
\caption{Comparison of the average MSE on the test set between the ground truths and the original noisy data, the reconstructions given by Dictionary Learning, and by our supervised strategy.}
    \label{tab:1}
\end{table}

\begin{figure}
    \centering
    \includegraphics[width=0.9\linewidth]{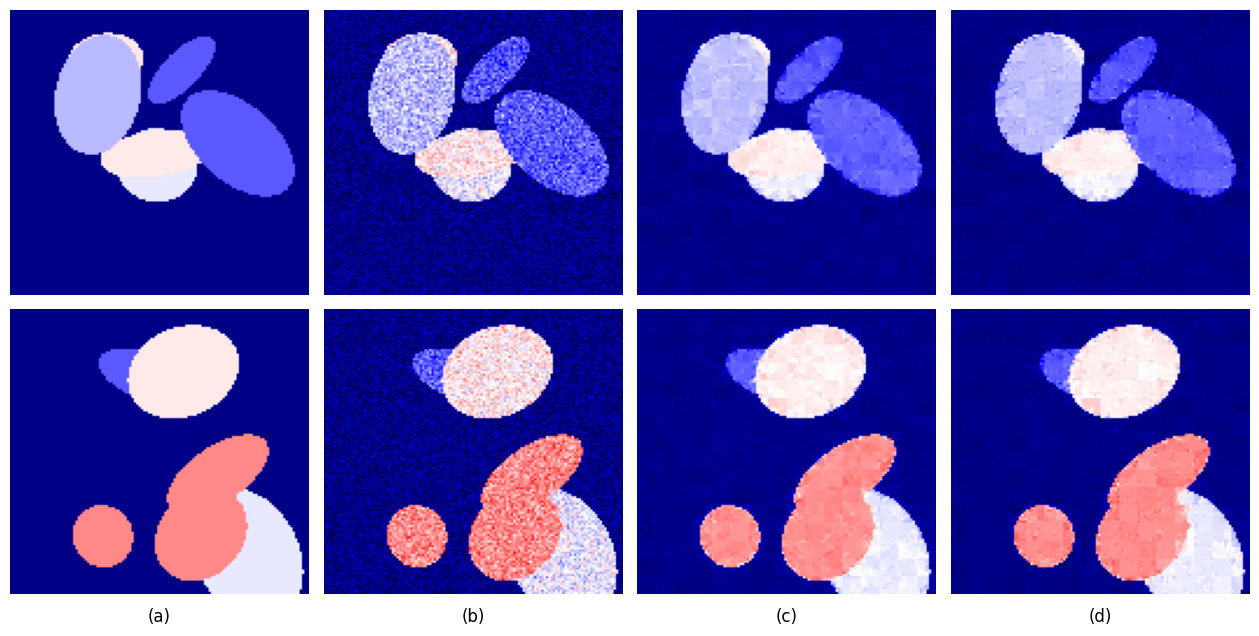}
    \caption{Visual comparison between some reconstructions of the images of the test set. (a): ground truths; (b): noisy images; (c): denoised images via Dictionary Learning; (d): denoised images via our strategy.}
    \label{fig:3}
\end{figure}

\paragraph{Example 3: a 1D-deblurring problem} 

We now consider the same 1D signals $\{x_j\}_{j=1}^m$ as in Example 1, but generate the measurements $\{y_j\}_{j=1}^m$ according to a different model: namely, $y_j = k \ast x_j + \eps_j$, being $\ast$ the 1D discrete convolution operation with reflecting boundary conditions, $k$ a Gaussian filter of standard deviation $\sigma_{\rm blur} = 1.2$, and $\epsilon_j \sim \mathcal{N}(0,\sigma^2 I)$ i.i.d. with $\sigma = 0.05$. This is a deblurring inverse problem, where the forward operator $A \in \R^{n\times n}$ is the circulant matrix associated with the filter $k$. We apply the relaxation-based algorithm proposed in Section \ref{subsec:algo}, choosing a small relaxation parameter, namely $\nu=10^{-3}$. The size of the training set is $m=20000$. We employ the Adam algorithm, initialized with a random orthogonal matrix, and leverage the explicit expression of the gradient of the loss, specifying a maximum number \texttt{n\_epochs} = 5000 of iterations, a stepsize \texttt{learning\_rate} = $1$, and $\mu_{\rm reg}=0.03$. The results of the training process are reported in Figure~\ref{fig:4}, where we show two examples of signals from the test set, together with their blurry, noisy observations and the reconstructions via our learned operator $R_{\widehat{B}}$. In the same figure, we report the learned matrix $\widehat{B}$: as expected, the learned basis $\widehat{B}$ is just a scaled version of the canonical basis (the identity matrix), but with its columns randomly shuffled (permuted) and flipped (sign swaps). This is the natural result because the signals $x_j$ were designed to be sparse in that exact basis. The learned scaling factor simply adjusts the effective strength of the $\ell^1$ regularization.
\begin{figure}
    \centering
    \includegraphics[width=0.99\linewidth]{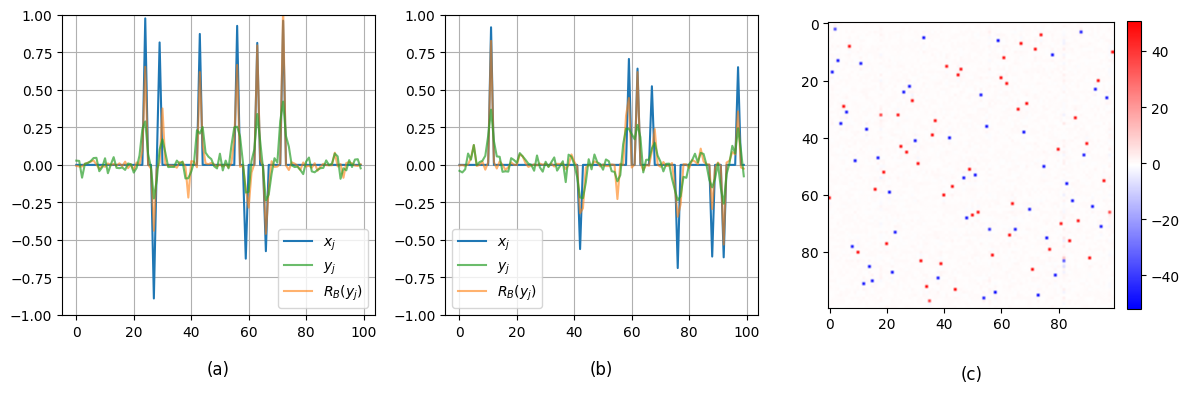}
    \caption{Results of the proposed technique for the 1D deblurring problem. Columns (a) and (b): comparison between the ground truth signals $x_j$, their blurred and noisy versions $y_j$, and their reconstructions $R_B(y_j)$ for two different examples. Column (c): the learned matrix $\widehat{B}$.}
    \label{fig:4}
\end{figure}

As in Example 1, we look for a benchmark $B$ based on the ansatz $B^\star = \alpha^\star I$, selecting $\alpha^\star$ via line-search. In Table \ref{tab:2}, we compare the values of $L(\widehat{B})$ and $L(B^\star)$ (approximated on a sufficiently large test set). We observe that our technique achieves a comparable accuracy without requiring any prior information on the basis with respect to which the signals are sparse.
\begin{table}[h!]
    \centering
\begin{tabular}{l|c|c|c}
 & Noisy, blurry data & Optimal choice, $B^\star$ & Our Method, $\widehat{B}$ \\ \hline
MSE & 3.4108E-02 & \textbf{1.7026E-02} & 1.7312E-02 \\ 
\hline
\end{tabular}
    \caption{Comparison of the average MSE on the test set between the ground truths and the original noisy and blurry data, the reconstruction given by $B^\star$, and our supervised strategy.}
    \label{tab:2}
\end{table}

\begin{appendices}

\section{White noise}\label{app:white_noise}

As we mentioned in Section~\ref{sec:intro}, 
the assumption on $\Se$ does not allow us, in principle, to consider white noise in infinite dimensions, for which the covariance $\Se=\mathrm{Id}$ is not trace-class, but only bounded. However, this situation can be considered by embedding $Y$ into a larger space. Instead of providing a general and abstract construction, we prefer to discuss two particular examples.  For more details, see \cite[Section~A.1]{alberti2021learning}.

\begin{example}\label{white_noise}
Let $\Omega$ be a connected bounded open subset of $\R^d$, and $\{\eps_y\}_{y\in L^2(\Omega)}$ be a Hilbert random process such that 
\[ \E{[\eps_y]}=0\qquad 
\E{[\eps_y \eps_{y'}]}=\<y,y'\>_{L^2(\Omega)}, \qquad y,y'\in L^2(\Omega),
\]
which is the classical model for white noise, since the covariance operator of $\eps$ is the identity. Assume that $\Omega$ satisfies a cone condition and fix $s>d/2$. Then the Rellich-Kondrachov theorem \cite[Theorem 6.3]{adams2003sobolev} implies that the Sobolev space $H^s(\Omega)$ is embedded into $C_b(\Omega)$, the space of bounded continuous functions on $\Omega$, and
the canonical inclusion $\iota\colon H^s(\Omega)\to L^2(\Omega)$ is a compact operator. The above embedding implies that $H^s(\Omega)$ is a reproducing kernel Hilbert space with bounded reproducing kernel $K\colon\Omega\times \Omega\to \R$. An easy calculation shows that
\[ \tr(\iota^*\circ\iota)=\int_U K(x,x) dx <+\infty.\]
This fact implies that in the  Gelfand triple
\begin{equation}
    H^s(\Omega)\xrightarrow[]{\;\;\iota\;\;} L^2(\Omega) \xrightarrow[]{\;\;\iota^*\;\;} H^{-s}(\Omega),\label{gelfand}
\end{equation}
both $\iota$ and $\iota^*$ are Hilbert-Schmidt operators.
Hence, it is possible to define a random variable $\widehat{\eps}$ taking values in $H^{-s}(\Omega)$ such that 
\[ \<\widehat{\eps},y\>_{H^{-s},H^s}=\eps_{\iota(y)},\qquad y\in H^s(\Omega).  \]
It is immediate to check that 
\[
\E{[ \<\widehat{\eps},y\>_{H^{-s},H^s}\ \<\widehat{\eps},y'\>_{H^{-s},H^s}]}=\<\iota(y),\iota(y')\>_{L^2(\Omega)},
\]
so that the covariance operator of $\widehat{\eps}$ is $\iota^*\circ\iota$, which is a trace class operator. Thus,  $\widehat{\eps}$ is a square-integrable random variable in $H^{-s}(\Omega)$. To apply the results of our paper, it is enough to set $Y=H^{-s}(\Omega)$ and to lift the inverse problem \eqref{eq:invprob} to $H^{-s}(\Omega)$:
\[
\iota^*(y) = (\iota^*\circ A)x + \widehat\varepsilon.
\]
This requires identifying $H^{s}(\Omega)$ and $H^{-s}(\Omega)$ using the Riesz lemma. Note that this identification is not standard, since it is not compatible with the double embedding of~\eqref{gelfand}. However, the intermediate space $L^2(\Omega)$ does not matter once $\widehat{\eps}$ is defined. 
\end{example}

In the next example, we consider the sequence space $\ell^2$, as a prototypical Hilbert space (once an orthonormal basis has been fixed).

\begin{example}
Let $X$ be a Hilbert space and set $Y=\ell^2$. Let $A\colon X\to\ell^2$ be a bounded and linear map, and consider the inverse problem \eqref{eq:invprob} 
\begin{equation}\label{eq:inv-white}
y = Ax+\varepsilon.
\end{equation}
Let $\{e_n\}_{n\in\N}$ be the canonical basis of $\ell^2$, defined by $(e_n)_i = \delta_{i,n}$. We consider the case when $\varepsilon$ is a white Gaussian noise with variance $\sigma^2$, namely,
\begin{equation*}
\varepsilon = \sum_{n\in\N}\varepsilon_n e_n,
\end{equation*}
where the  random variables $\varepsilon_n$ are i.i.d.\ scalar Gaussians with mean zero and variance $\sigma^2$, i.e.\  $\varepsilon_n\sim\mathcal{N}(0,\sigma^2)$. The above expression for $\varepsilon$ is only formal, because the series is divergent in $\ell^2$ with probability 1. As a consequence, \eqref{eq:inv-white} is not well-defined in $\ell^2$. However, writing \eqref{eq:inv-white} in components with respect to the orthonormal basis $\{e_n\}_{n\in\N}$ yields 
\begin{equation}\label{eq:family}
    y_n = (Ax)_n + \varepsilon_n,\qquad n\in\N,
\end{equation}
    where we wrote $Ax = \sum_n (Ax)_n e_n$. This is a family of well-defined scalar equations. 
    
Let us now see how it is possible to reformulate this as a problem in a Hilbert space.    
    Equivalently, we can rewrite \eqref{eq:family} as
    \begin{equation}\label{eq:family2}
    \frac{y_n}{n^s} = \frac{(Ax)_n}{n^s} + \frac{\varepsilon_n}{n^s},\qquad n\in\N,
\end{equation}
for some $s>\frac12$. Let us introduce the embedding
$
\iota^*\colon\ell^2\to\ell^2$ defined by $\iota^*(e_n)= e_n/n^s
$
and the  random variable
\begin{equation*}
\widehat\varepsilon = \sum_{n\in\N}\frac{\varepsilon_n}{n^s} e_n.
\end{equation*}
Note that $\E{[\norm{\widehat\eps}^2_2]}<+\infty$, and $\widehat\varepsilon\in \ell^2$ with probability 1. Thus, we can 
 rewrite \eqref{eq:family2} as
\[
\iota^*(y) = \iota^*(Ax) + \widehat\varepsilon.
\]
This equation is meaningful in $\ell^2$, and has the same form as the original inverse problem \eqref{eq:inv-white}. 
\end{example}

\section{Connections with Dictionary Learning and unsupervised strategies} \label{app:DL}

Although our approach shares the same aim of dictionary learning, i.e., promoting the sparse representation of some ground truths by selecting a suitable synthesis operator, it is possible to outline some substantial differences. The key observation is that the optimal operator sought in dictionary learning is independent of the forward operator and the noise distribution, and depends only on the distribution of $x$. On the contrary, the optimal $\BS$ in \eqref{eq:bilevel} depends on both $\varepsilon$ and $A$, yielding, in general, a smaller MSE and, consequently, better statistical guarantees for the solution of the inverse problem. 

Let us briefly introduce the standard dictionary learning framework \cite{tosic-frossard}.
If, instead of the training dataset $\vzv = \{(x_j,y_j)\}_{j=1}^m$ employed in \eqref{eq:empiricaltarget} to discretize \eqref{eq:bilevel}, only a collection of ground truths $\{x_j\}_{j=1}^m$ is available, i.i.d. sampled from the (unknown) marginal $\rho_x$, we may consider the following unsupervised technique: let
$\wt{R}_B$ a sparsity-promoting regularizer of the form \eqref{eq:xhat}-\eqref{eq:RB-def} associated with $A = \operatorname{Id}$ and without assuming the knowledge of the covariance $\Se$, namely:
\begin{gather*}
\wt{R}_B(x) = B \wt{u}_B(x),  \quad \wt{u}_B(x) = \argmin_{u \in \ell^1} \left\{ \frac{1}{2} \|Bu - x \|^2 + \|u\|_{\ell^1}\right\} \\  
    \BS_{DL} \in \argmin_{B \in \mathcal{B}} \left\{ \frac{1}{m} \sum_{j=1}^m \| x_j - \wt{R}_B(x_j) \|^2 \right\}. \label{eq:DL_out}
\end{gather*}

This problem yields a bilevel formulation of the well-known Dictionary Learning problem \cite{peyre-fadili,yang-et-al}.
Our supervised strategy resembles dictionary learning, indeed, let us recall that, according to \eqref{eq:empiricaltarget},
\begin{align*}
    \BS &\in \argmin_{B \in \mathcal{B}} \left\{ \frac{1}{m} \sum_{j=1}^m \| x_j - R_B(A x_j+\varepsilon_j) \|^2 \right\},
\end{align*}
being $R_B$ as in \eqref{eq:RB-def}. However, since $R_B$ is a nonlinear map, differently from our previous work on quadratic regularizers \cite{alberti2021learning}, it is easy to show that the two problems are not equivalent in general. In particular, while $\BS_{DL} $ is independent of the forward operator $A$ and on the noise $\varepsilon$ by construction,  $\BS$ will in general depend on both. We illustrate this with a simple 1D example.

Let $\sigma^2=\mathbb{E}(\varepsilon^2)>0$ denote the variance of the noise (with zero mean) and consider the 1D problem
\[
A x = ax \quad \text{with } a \in (0,\sigma]
\]
and the regularization $B^{-1}x = bx$, $b>0$.  Given an unknown $x^\dagger$ and data $y= Ax^\dagger + \varepsilon = ax^\dagger + \varepsilon$, the Lasso reconstruction is given by
\[
\hat x = \argmin_x \frac12|\sigma^{-1}(A x - y)|^2 +  |bx|
= \argmin_x \frac12| x - y/a|^2 +  \frac{\sigma^2 b}{a^2}|x|.
\]
Setting $\gamma = \frac{a}{\sigma}$, this may be rewritten by using the soft-thresholding operator
 $S_\lambda$ as
 \[
\hat x = S_{\frac{b}{\gamma^2}}\left(x^\dagger +\gamma^{-1}\tilde \varepsilon\right),
 \]
 where $\tilde\varepsilon = \varepsilon/\sigma$ satisfies $\mathbb{E}(\tilde \varepsilon^2)=1$.
Now consider the mean squared error
\[
\text{MSE} = \mathbb{E}_{x, \tilde\varepsilon}\left[\left|S_{\frac{b}{\gamma^2}}\left(x +\gamma^{-1}\tilde \varepsilon\right) -x\right|^2\right].
\]
Note that we consider the minimization of the expected risk, and not of the empirical risk, as it is more significant.
For simplicity, we choose the following zero-mean independent distributions
\[
\mathbb P(x = \pm 1) = \frac{1}{2}, \quad \mathbb P (\tilde\varepsilon = \pm 1) = \frac{1}{2}.
\]
After a series of elementary computations, we can show that 
\begin{equation*}
    \text { MSE }= \begin{cases}\frac{1}{\gamma^4}( b-\gamma )^2 & b \in\left(0, \gamma-\gamma^2\right), \\ \frac{1}{2}\left[1+\frac{1}{\gamma^4}(  b-\gamma)^2\right] & b \in\left[\gamma-\gamma^2, \gamma+\gamma^2\right), \\ 1 & b \geqslant \gamma+\gamma^2.\end{cases}
\end{equation*}
Therefore, the optimal value for the MSE is achieved at $b=\gamma =\frac{a}{\sigma}$ and depends on both $a$ and $\sigma$, namely, on both the forward operator and the noise level. Any unsupervised choice for $b$ would be independent of $a$ and $\sigma$ and give rise, in general, to larger values of the MSE.

\section{Proofs of Section~\ref{sec:perturbation_known}}\label{sec:proofs_example}

\begin{proof}[Proof of Lemma~\ref{lem:Bcompactexample}]
    Every element $B\in\mathcal{B}$ satisfies Assumption~\ref{ass:FBI} because $A$ is injective and $B|_{\ell^2_I}$ is injective for every finite subset $I\subset\N$ by \eqref{eq:Hexample}.
It remains to show that $\mathcal B$ is compact. Since the map 
    \begin{equation}\label{eq:map-lipschitz}
        \LO(\ell^2)\ni U \longmapsto B_0 (\operatorname{Id} + U) \in \LO(\ell^2,X)
    \end{equation}
   is continuous, it is enough to show that
    $\mathcal{H}\subset\LO(\ell^2)$
    is compact. Write $\mathcal{H}=\mathcal{K}\cap\mathcal{C}$, where
    \begin{align*}
        &\mathcal{K} = \{ E_1 T E_2 :  T \in \LO(\ell^2), \; \| T \|_{\LO(\ell^2)}\leq 1\},\\
        &\mathcal{C} =\{K\in\LO(\ell^2): \|(\operatorname{Id} + K)u\|\ge c_I \|u\| \; \text{ for every finite $I\subset\N$ and $u\in \ell^2_I$}\}.
    \end{align*}
The set $\mathcal{K}$ is compact by \cite[Theorem 3]{vala1964}, because  $\left\{   T \in \LO(\ell^2):  \| T \|_{\LO(\ell^2)}\leq 1 \right\}$ is closed.  The set $\mathcal{C}$ is closed, because convergence in operator norm implies pointwise convergence. Hence, $\mathcal{H}=\mathcal{K}\cap\mathcal{C}$ is compact.
\end{proof}

\begin{proof}[Proof of Proposition~\ref{prop:coveringexample}]
    Note that
    \begin{equation}
    \mathcal{N}(\mathcal{B},r; \|\cdot \|_{\LO(\ell^2,X)}) \le \mathcal{N}(\mathcal{H},r;\|\cdot \|_{\LO(\ell^2)}),
        \label{eq:B_step1}
    \end{equation}
since the map in \eqref{eq:map-lipschitz} is Lipschitz with constant $1$ (recall that $\| B_0 \|_{\LO(\ell^2,X)}\leq 1$). Using the notation of the proof of Lemma~\ref{lem:Bcompactexample}, we have $\mathcal{H}=\mathcal{K} \cap \mathcal{C}$, so that $\mathcal{H}\subseteq\mathcal{K}$. Thus
    \begin{equation}
        \mathcal{N}(\mathcal{H},r;\|\cdot \|_{\LO(\ell^2)}) \leq \mathcal{N}(\mathcal{K},r;\|\cdot \|_{\LO(\ell^2)}).
        \label{eq:B_step2}
    \end{equation}

    To further bound the right-hand side of \eqref{eq:B_step2}, we consider the singular value decomposition of $E$. Since $E$ is self-adjoint and compact, we can write its spectral decomposition in terms of its eigenvalues $\{\lambda_n\}$ and eigenvectors $\{e_n\}$, and define its rank-$N$ approximation $E_N$ as follows:
    \begin{equation}\label{eq:EE_N}
    E = \sum_{n=1}^{\infty} \lambda_n e_n \otimes e_n, \quad \qquad E_N = \sum_{n=1}^{N} \lambda_n e_n \otimes e_n=P_NE=EP_N,
\end{equation}
    where we denote by $u \otimes v$ the rank-$1$ operator such that $(u \otimes v)x = \lX v,x \rangle_{\ell^2} u$, being $\lX \cdot, \cdot \rangle_{\ell^2}$ the inner product in $\ell^2$, and
    \(
    P_N=\sum_{n=1}^{N}  e_n \otimes e_n
    \)
    is the orthogonal projection onto $\operatorname{span}\{e_1,\dots,e_N\}$.
    Assuming that the sequence $\{\lambda_n\}$ is decreasingly ordered, we know that $\| E - E_N \|_{\LO(\ell^2)} \leq \lambda_{N+1}$. We now introduce the spaces
    \[
    \mathcal{K}_N = \left\{ K = E_N T E_N: T \in \LO(\ell^2), \ \| T \|_{\LO(\ell^2)}\leq 1 \right\},
    \]
    and prove that the following bound holds: 
    \begin{equation}
    \mathcal{N}(\mathcal{K}, \rho + 2\lambda_{N+1}; \| \cdot \|_{\LO(\ell^2)}) \leq \mathcal{N}(\mathcal{K}_N,\rho; \| \cdot \|_{\LO(\ell^2)})=:{\hat{\mathcal{N}}} .
        \label{eq:B_step3}
    \end{equation}
        In order to prove \eqref{eq:B_step3},  consider a $\rho$-covering $\{K_1, \ldots, K_{\hat{\mathcal{N}}}\}$ of $\mathcal{K}_N$. For any $K = ETE\in\mathcal{K}$, let $K^N = E_N T E_N\in\mathcal{K}_N$, and let $K_i$ be such that $\| K^N - K_i\| \leq \rho$. Then,
    \[ 
    \begin{aligned}
    \| K - K_i \|_\LO &\leq \| K - K^N \|_\LO + \|K^N - K_i \|_\LO \\
    & \leq \| ETE - ETE_N \|_\LO + \| ETE_N - E_N T E_N \|_\LO + \rho \\
    & \leq \| E\|_\LO \| T \|_\LO \| E - E_N\|_\LO + \| E-E_N\|_\LO \| T \|_\LO  \| E_N \|_\LO + \rho \\
    & \leq 2 \lambda_{N+1} + \rho,
    \end{aligned}
    \]
    which shows that $\{K_1, \ldots, K_{\hat{\mathcal{N}}}\}$ is a $(2 \lambda_{N+1} + \rho)$-covering of $\mathcal{K}$.

    In order to bound the covering numbers of $\mathcal{K}_N$, we claim that $\mathcal{K}_N \subset \mathcal{F}_N$, where
    \[
    \mathcal{F}_N = \left\{ K = E T E: T \in \operatorname{HS}(\ell^2), \ \| T \|_{\operatorname{HS}(\ell^2)}\leq \sqrt{N} \right\},
    \]
    and $\operatorname{HS}(\ell^2)$ denotes the class of Hilbert-Schmidt operators from $\ell^2$ to $\ell^2$, namely, the ones for which the singular values are square-summable. In order to show this, take $K \in \mathcal{K}_N$, with $K\!=\!E_N T E_N$ for some $T\!\in\!\LO(\ell^2)$ such that $\| T\|_{\LO(\ell^2)}\!\leq\!1$. 
    Setting
    $T_N\!=\!P_N T P_N$, by \eqref{eq:EE_N} we get
\[
K = E_NTE_N = E P_N T P_N E = E T_N E.
\]
    Note that $T_N$ is a rank-$N$ operator, thus belonging to the Hilbert-Schmidt class. Furthermore,
    \[
    \|T_N\|_{\operatorname{HS}(\ell^2)} = \|P_N T P_N\|_{\operatorname{HS}(\ell^2)}\le \|P_N\|_{\operatorname{HS}(\ell^2)} \|T P_N\|_{\LO(\ell^2)}\le \sqrt{N},
    \]
    because $\|P_N\|_{\operatorname{HS}(\ell^2)}=\sqrt{N}$, $\|T P_N\|_{\LO(\ell^2)}\le 1$ and  $$\|P_N T P_N\|^2_{\operatorname{HS}(\ell^2)} = \sum_{n=1}^\infty \|P_N T P_N e_n \|^2_{\ell^2} \leq \sum_{n=1}^\infty \| P_N T \|^2_{\mathcal{L}(\ell^2)} \| P_N e_n\|_{\ell^2}^2 = \| P_N T \|^2_{\mathcal{L}(\ell^2)} \| P_N \|_{\operatorname{HS}(\ell^2)}^2. $$ This shows that $K\in \mathcal{F}_N$, as claimed. By a simple scaling argument, we have 
    \begin{equation} \label{eq:B_step4}
        \mathcal{N}(\mathcal{K}_N,\rho;\| \cdot \|_\LO) \leq \mathcal{N}(\mathcal{F}_N,\rho; \| \cdot \|_\LO) = \mathcal{N}(\mathcal{F}_1,\rho N^{-1/2}; \| \cdot \|_\LO).
    \end{equation}
    
    We finally have to estimate the covering numbers $\mathcal{N}(\mathcal{F}_1,\varrho;\| \cdot \|_\LO)$. Let us define
    \[
    F_1 = \left\{ T \in \operatorname{HS}(\ell^2): \ \| T \|_{\operatorname{HS}(\ell^2)}\leq 1 \right\},
    \] 
    which entails that $\mathcal{F}_1 = j(F_1)$, where $j$ is the (compact) embedding $j\colon \operatorname{HS}(\ell^2) \rightarrow \operatorname{HS}(\ell^2)$ defined as $j( T) = ETE$.
  Since $\| \cdot \|_{\LO(\ell^2)} \le \| \cdot \|_{\operatorname{HS}(\ell^2)}$, a $\varrho-$covering of $\mathcal{F}_1$ with respect to $\| \cdot \|_{\operatorname{HS}(\ell^2)}$ is also a $\varrho$-covering with respect to $\| \cdot \|_{\mathcal{L}(\ell^2)}$. Thus,
    \begin{equation}\label{eq:boh}
    \mathcal{N}(\mathcal{F}_1,\varrho;\| \cdot \|_{\LO(\ell^2)}) \leq \mathcal{N}(\mathcal{F}_1,\varrho; \| \cdot \|_{\operatorname{HS}(\ell^2)}).
    \end{equation}
    The quantity $\mathcal{N}(\mathcal{F}_1,\varrho; \| \cdot \|_{\operatorname{HS}(\ell^2)})$
    is the covering number of the image of the unit sphere $F_1$ of the Hilbert space $\operatorname{HS}(\ell^2)$ through the embedding $j$, and is linked to the \textit{entropy numbers} $\varepsilon_k(j)$: indeed, according to the definitions in \cite[Chapter 1]{cast90}, one clearly sees that 
    \[
    \mathcal{N}(\mathcal{F}_1, \varrho; \| \cdot \|_{\operatorname{HS}(\ell^2)}) \leq k \quad \iff \quad \varepsilon_k(j) \leq \varrho.
    \]
    In order to estimate the entropy numbers of $j$, we can rely on its singular values, thanks to \cite[Theorem~3.4.2]{cast90}. In view of the decay $\lambda_n \lesssim n^{-s}$ of the eigenvalues of $E$, following the proof of \cite[Lemma A.9]{alberti2021learning} we can show that the singular values of $j$ decay as $n^{-s'}$, for any $s'< s$. Then, by the same argument used in the proof of \cite[Lemma A.8]{alberti2021learning}, we have that $\varepsilon_k(j) \lesssim (\log k)^{-s'}$, which implies $\mathcal{N}(\mathcal{F}_1, (\log k)^{-s'}; \| \cdot \|_{\operatorname{HS}(\ell^2)})\leq k$ and ultimately
    \begin{equation}
        \log \mathcal{N}(\mathcal{F}_1, \varrho; \| \cdot \|_{\operatorname{HS}(\ell^2)}) \lesssim \varrho^{-1/s'}  .
        \label{eq:B_step5}
    \end{equation}
    
    We can now conclude the proof: by \eqref{eq:B_step1}, \eqref{eq:B_step2} and \eqref{eq:B_step3}, setting $\rho = \frac{r}{2}$ and choosing $N$ such that $\lambda_{N+1} = \frac{r}{4}$ (which implies $(N+1)^{-s} \gtrsim \frac{r}{4}$, hence $N \lesssim r^{-1/s}$), we get
    \[
    \mathcal{N}(\mathcal{B},r; \|\cdot \|_{\LO(\ell^2,X)}) \lesssim \mathcal{N}(\mathcal{K}_{N},r;\| \cdot\|_{\LO(\ell^2)}).
    \]
    Next, by \eqref{eq:B_step4}, we obtain
    \[
     \mathcal{N}(\mathcal{B},r; \|\cdot \|_{\LO(\ell^2,X)}) \lesssim \mathcal{N}(\mathcal{F}_1, r N^{-1/2}; \| \cdot \|_{\LO(\ell^2)}),
    \]
    and by  \eqref{eq:boh} and \eqref{eq:B_step5}  we finally
    obtain 
    \[
     \log\mathcal{N}(\mathcal{B},r; \|\cdot \|_{\LO(\ell^2,X)}) \lesssim \bigl(rN^{-1/2}\bigr)^{-1/s'}
     \lesssim r^{-\frac{2s+1}{2ss'}}.
    \]
    This shows   \eqref{eq:covering_example_comp}, and the proof is concluded.
\end{proof}

\end{appendices}
\section*{Acknowledgments}
This material is based upon work supported by the Air Force Office of Scientific Research under award
FA8655-23-1-7083 and was cofunded by the European Union (ERC, SAMPDE, 101041040, ERC, SLING, 819789,
ERC, AdG project 101097198, and Next Generation EU - project FAIR, PE0000013, spoke 8). Views and opinions
expressed are, however, those of the authors only and do not necessarily reflect those of the European Union or the
European Research Council. GSA, EDV, LR and MS are members
of INdAM (Istituto Nazionale di Alta Matematica). TH and ML were supported by the
Research Council of Finland, decision numbers 348504, 353094, 359182 and 359183 (FAME flagship).

\bibliographystyle{siamplain}
\bibliography{references}

\end{document}